\documentclass{article}

\usepackage[square]{natbib}

\usepackage{microtype}
\usepackage{graphicx}
\usepackage{subcaption}
\usepackage{booktabs}

\usepackage[preprint]{icml2026}
\usepackage{hyperref}
\usepackage{amsmath}
\PassOptionsToPackage{table}{xcolor}
\usepackage{lipsum,booktabs}
\usepackage{amsmath,mathrsfs,amssymb,amsfonts,bm,enumitem}
\usepackage{rotating}
\usepackage{pdflscape}
\usepackage{tcolorbox}
\usepackage{url,dsfont,nicefrac}
\usepackage{multirow}
\usepackage{makecell}
\usepackage{anyfontsize}
\usepackage{bbding}
\usepackage{xspace}
\usepackage{enumitem}
\hypersetup{
    colorlinks,
    breaklinks,
    linkcolor = blue,
    citecolor = blue,
    urlcolor  = blue,
}
\allowdisplaybreaks
\usepackage{appendix}
\usepackage{multirow,makecell,tabularx}
\usepackage{xfrac}
\usepackage{nicefrac}
\usepackage{threeparttable}
\usepackage{pifont}
\usepackage{tikz}
\usepackage{wrapfig}
\usepackage{nicefrac}
\usepackage{algorithmic,algorithm}
\usepackage{accents}

\setlist[itemize]{leftmargin=5.5mm}

\definecolor{MyWineRed}{RGB}{200, 20, 20}
\definecolor{MyDarkGreen}{RGB}{0, 160, 0}
\newcommand{\cmark}{{\color{black}\ding{51}}}%
\newcommand{\xmark}{{\color{black}\ding{55}}}%

\newcommand{\cxmark}{{\color{black}\ding{51}\kern-0.642em\ding{55}}}

\renewcommand{\tilde}{\widetilde}
\renewcommand{\hat}{\widehat}

\def \A {\mathcal{A}}

\def \E {\mathbb{E}}

\def \F {\mathcal{F}}

\def \K {\mathcal{K}}

\def \O {\mathcal{O}}

\def \Q {\mathcal{Q}}

\def \R {\mathbb{R}}
\def \S {\mathcal{S}}

\def \X {\mathcal{X}}

\def \g {\mathbf{g}}

\def \x {\mathbf{x}}
\def \y {\mathbf{y}}

\def \gh {\hat{\g}}

\def \xh {\hat{\x}}
\def \xb {\overline{\x}}
\def \xt {\tilde{\x}}

\def \Ot {\tilde{\O}}

\def \ellh {\widehat{\ell}}
\def \ellt {\widetilde{\ell}}

\def \is {i_{\star}}
\def \xs {\x^{\star}}

\def \sumT {\sum_{t=1}^T}
\def \sumt {\sum_{s=1}^t}

\def \define {\triangleq}

\let \eps \epsilon
\def \epsilon {\varepsilon}

\usepackage{mathtools}
\let\norm\undefined 
\DeclarePairedDelimiter\norm{\lVert}{\rVert}
\DeclarePairedDelimiter\abs{\lvert}{\rvert}
\newcommand\inner[2]{\langle #1, #2 \rangle}

\newcommand\sbr[1]{\left( #1 \right)}
\newcommand\mbr[1]{\left[ #1 \right]}
\newcommand\bbr[1]{\left\{ #1 \right\}}

\def \Reg {\textsc{Reg}}

\DeclareMathOperator*{\poly}{poly}
\DeclareMathOperator*{\argmax}{arg\,max}
\DeclareMathOperator*{\argmin}{arg\,min}

\newcommand\given[1][]{\:#1\vert\:}

\usepackage{amsthm}
\newtheorem{myThm}{Theorem}

\newtheorem{myLemma}{Lemma}

\theoremstyle{definition}
\newtheorem{myAssum}{Assumption}

\newtheorem{myDef}{Definition}

\newtheorem{myRemark}{Remark}

\newtheoremstyle{proofsketchstyle}{}{} {} {} {\itshape} {.}{ } {\thmname{#1}\thmnote{ #3}} 

\theoremstyle{proofsketchstyle}

\newcounter{romancounter}
\newcommand{\rom}[1]{\setcounter{romancounter}{#1}\textit{(\roman{romancounter})}}

\usepackage{graphicx,color}

\definecolor{myred}{RGB}{192,0,1}

\definecolor{myblue}{RGB}{68,114,196}

\definecolor{mygreen}{RGB}{0,128,0}

\definecolor{myorange}{rgb}{0.84, 0.40, 0.15}

\usepackage{prettyref}
\newcommand{\pref}[1]{\prettyref{#1}}

\newcommand{\savehyperref}[2]{\texorpdfstring{\hyperref[#1]{#2}}{#2}}
\newrefformat{eq}{\savehyperref{#1}{Eq.~\textup{(\ref*{#1})}}}
\newrefformat{eqn}{\savehyperref{#1}{Eq.~(\ref*{#1})}}
\newrefformat{lem}{\savehyperref{#1}{Lemma~\ref*{#1}}}
\newrefformat{lemma}{\savehyperref{#1}{Lemma~\ref*{#1}}}
\newrefformat{def}{\savehyperref{#1}{Definition~\ref*{#1}}}
\newrefformat{line}{\savehyperref{#1}{Line~\ref*{#1}}}
\newrefformat{thm}{\savehyperref{#1}{Theorem~\ref*{#1}}}
\newrefformat{table}{\savehyperref{#1}{Table~\ref*{#1}}}
\newrefformat{corr}{\savehyperref{#1}{Corollary~\ref*{#1}}}
\newrefformat{cor}{\savehyperref{#1}{Corollary~\ref*{#1}}}
\newrefformat{sec}{\savehyperref{#1}{Section~\ref*{#1}}}
\newrefformat{subsec}{\savehyperref{#1}{Section~\ref*{#1}}}
\newrefformat{app}{\savehyperref{#1}{Appendix~\ref*{#1}}}
\newrefformat{appendix}{\savehyperref{#1}{Appendix~\ref*{#1}}}
\newrefformat{assum}{\savehyperref{#1}{Assumption~\ref*{#1}}}
\newrefformat{ex}{\savehyperref{#1}{Example~\ref*{#1}}}
\newrefformat{fig}{\savehyperref{#1}{Figure~\ref*{#1}}}
\newrefformat{alg}{\savehyperref{#1}{Algorithm~\ref*{#1}}}
\newrefformat{remark}{\savehyperref{#1}{Remark~\ref*{#1}}}
\newrefformat{conj}{\savehyperref{#1}{Conjecture~\ref*{#1}}}
\newrefformat{prop}{\savehyperref{#1}{Proposition~\ref*{#1}}}
\newrefformat{proto}{\savehyperref{#1}{Protocol~\ref*{#1}}}
\newrefformat{prob}{\savehyperref{#1}{Problem~\ref*{#1}}}
\newrefformat{claim}{\savehyperref{#1}{Claim~\ref*{#1}}}
\newrefformat{que}{\savehyperref{#1}{Question~\ref*{#1}}}
\newrefformat{op}{\savehyperref{#1}{Open Problem~\ref*{#1}}}
\newrefformat{fn}{\savehyperref{#1}{Footnote~\ref*{#1}}}
\newrefformat{require}{\savehyperref{#1}{Requirement~\ref*{#1}}}

\usepackage{array} 
\usepackage{tabularx} 

\definecolor{colorMain}{HTML}{16A085}   
\definecolor{colorVar_1}{HTML}{2980B9}    
\definecolor{colorVar_2}{HTML}{16A085}   
\definecolor{colorSmall}{rgb}{0.45, 0.45, 0.45}
\definecolor{colorErr}{HTML}{5B2C6F}

\def \var {\textsc{Var}}
\def \UniXGrad {\textsc{UniXGrad}\xspace}
\def \Holder {Hölder\xspace}
\def \udog {\textsc{U-DoG}\xspace}
\def \out {\textnormal{out}}

\begin{document}

\twocolumn[
    \icmltitle{Towards Fully Parameter-Free Stochastic Optimization:\\ Grid Search with Self-Bounding Analysis}

    \icmlsetsymbol{equal}{*}

    \begin{icmlauthorlist}
        \icmlauthor{Yuheng Zhao}{NJU_keylab,NJU_AI}
        \icmlauthor{Yu-Hu Yan}{NJU_keylab,NJU_AI}
        \icmlauthor{Amit Attia}{Tel}
        \icmlauthor{Tomer Koren}{Tel}
        \icmlauthor{Lijun Zhang}{NJU_keylab,NJU_AI}
        \icmlauthor{Peng Zhao}{NJU_keylab,NJU_AI}
    \end{icmlauthorlist}

    \icmlaffiliation{NJU_keylab}{National Key Laboratory for Novel Software Technology, Nanjing University, China}
    \icmlaffiliation{NJU_AI}{School of Artificial Intelligence, Nanjing University, China}
    \icmlaffiliation{Tel}{Tel Aviv University}

    \icmlcorrespondingauthor{Peng Zhao}{zhaop@lamda.nju.edu.cn}

    \icmlkeywords{}

    \vskip 0.3in
]
\printAffiliationsAndNotice{}

\begin{abstract}
    Parameter-free stochastic optimization aims to design algorithms that are agnostic to the underlying problem parameters while still achieving convergence rates competitive with optimally tuned methods.
    While some parameter-free methods do not require the specific values of the problem parameters, they still rely on prior knowledge, such as the lower or upper bounds of them. 
    We refer to such methods as \mbox{``partially parameter-free''}.
    In this work, we target achieving ``\emph{fully} parameter-free'' methods, i.e., the algorithmic inputs do not need to satisfy any \emph{unverifiable} condition related to the true problem parameters.
    We propose a powerful and general \emph{grid search} framework, named \textsc{Grasp}, with a novel \mbox{\emph{self-bounding}} analysis technique that effectively determines the search ranges of parameters, in contrast to previous work.
    Our method demonstrates generality in:
    \rom{1} the non-convex case, where we propose a fully parameter-free method that achieves near-optimal convergence rate, up to logarithmic factors;
    \rom{2} the convex case, where our parameter-free methods are competitive with strong performance in terms of acceleration and universality.
    Finally, we contribute a sharper guarantee for the model ensemble, a final step of the grid search framework, under interpolated variance characterization.
\end{abstract}


\section{Introduction}
\label{sec:intro}
Stochastic optimization is a fundamental problem in machine learning and optimization~\citep{bottou2018optimization,Book:Lan-SCO}, with applications across various domains, such as the training of large-scale models~\citep{ICLR'14:Adam,ICLR'19:AdamW}, reinforcement learning~\citep{schulman2017DPO,shao2024GRPO}, and so on.

One of the most fundamental stochastic optimization methods is Stochastic Gradient Descent (SGD)~\citep{SGD}, which uses stochastic gradients estimated from mini-batch samples for the update.
The success of SGD depends largely on the choice of step size. In practice, a common approach is to define lower and upper bounds for the step size and perform a grid search within this range to find the optimal tuning. 
Typically, heuristic bounds such as $10^{-5}$ and $10^1$ are used, which lack theoretical guarantees.

Meanwhile, optimization theory suggests that the optimal step size depends on problem parameters, such as the gap between the initial point and the optimum, the properties of the objective function (e.g., smoothness coefficient, noise bound)~\citep{ghadimi2013stochastic}, etc.
Since these parameters are often difficult to access in practice, research has increasingly focused on \emph{parameter-free} stochastic optimization with notable recent developments~\citep{faw2022power,kavis2022high,attia2023sgd,liu2023high,ICML'24:Amit,khaled2024tuningfree,ICML'23:DoG,ICML'23:Defazio,ICML'23:DoG,COLT'24:Kreisler}, where convex and non-convex methods are developed without specific values of the problem parameters. However, despite aiming for parameter-freeness, most works still require inputs tied to certain problem parameters, e.g., the state-of-the-art method in stochastic convex optimization, \udog~\citep{COLT'24:Kreisler}, requires a lower bound for the initial distance to the optimum as an algorithmic input.

We introduce a more rigorous characterization of parameter-freeness to distinguish between ``\emph{partially}'' or ``\emph{fully}'' parameter-free methods, enabling clearer comparisons across existing approaches. Specifically, an algorithm is considered fully parameter-free with respect to a problem parameter $X$ if it is not only agnostic to the specific value of $X$ but also does not rely on any \emph{unverifiable} condition related to $X$ (e.g. upper or lower bounds). Otherwise, it is classified as  \emph{partially} parameter-free.

\begin{table*}[!t]
    \centering
    \caption{Comparison of parameter-freeness and convergence rates for stochastic optimization. 
    Notations: $\bar{X} \define \max\{X, X_\epsilon\}$ for $X \in \{L_\ell, F_\ell, L_\nu, d_0\}$ with \emph{any} $X_\epsilon>0$; 
    `\xmark', `\cxmark', and `\cmark' denote no, partial, and full parameter-freeness;
    $\delta$ is the confidence level.}
    \label{table:combined_comparison}

    \begin{subtable}{\textwidth}
        \centering
        \caption{\small{\textbf{Non-convex case.} $L_\ell$ is smoothness parameter, $F_\ell \define \ell(\x^0) - \ell(\xs)$, $\Delta_\ell$ is maximum gradient noise, $\hat{\mathbf{g}}^0$ is empirical value of $\nabla\ell(\x^0)$. The result of~\citet{ICML'24:Amit} is partially free because the lower and upper bounds of the learning rate, i.e., $\eta_{\min}$ and $\eta_{\max}$, require prior knowledge of the problem parameters.}}
        \renewcommand{\arraystretch}{1.3}
        \resizebox{0.9\textwidth}{!}{
        \begin{tabular}{c | *{3}{w{c}{0.6cm}} | c}
        \hline 
        
        \hline
        \multirow{2}{*}{\textbf{Reference}} & \multicolumn{3}{c|}{\textbf{Freeness}} & \multirow{2}{*}{\textbf{Convergence Rate of $\norm{\nabla \ell(\x^{\textnormal{out}})}$}} \\ \cline{2-4}
        & $L_\ell$ & $F_\ell$ & $\Delta_\ell$ & \\ \hline
        \rule{0pt}{5.5mm}
        Tuned SGD~\citep{ghadimi2013stochastic} & \xmark & \xmark & \xmark & $\sqrt{\frac{{L}_\ell{F}_\ell {\Delta}_\ell^2}{T} } + \frac{{L}_\ell{F}_\ell + \Delta_{\ell}^2(\log\frac{1}{\delta})}{T}$ \\[1mm] \hline
        \rule{0pt}{6mm}
        \citet[Theorem 1]{ICML'24:Amit} & \cxmark & \cxmark & \cxmark & $ \sqrt{\frac{L_{\ell} F_\ell \Delta_{\ell}^2}{T} \sbr{\log\frac{\eta_{\max }}{\eta_{\min }}} \sbr{\log\frac{1}{\delta}}} + \frac{L_{\ell} F_\ell+\Delta_{\ell}^2 (\log \frac{1}{\delta})}{T} \sbr{\log\frac{\eta_{\max }}{\eta_{\min }}} \sbr{\log\frac{1}{\delta}}$ \\[2mm] \hline
        \rule{0pt}{6mm}
        \textbf{Ours~[\pref{thm:non-convex}]} & \cmark & \cmark & \cmark & $\sqrt{\frac{\bar{L}_\ell\bar{F}_\ell {\Delta}_\ell^2}{T} \sbr{\log \frac{T \norm{\gh^0}}{L_\epsilon F_\epsilon}} } + \frac{\bar{L}_\ell\bar{F}_\ell + \Delta_{\ell}^2(\log\frac{1}{\delta})^2}{T}\sbr{\log \frac{T \norm{\gh^0}}{L_\epsilon F_\epsilon}}$ \\[2mm] \hline 
        
        \hline
        \end{tabular}
        }
        \label{table:non-convex}
    \end{subtable}

    \vspace{1em}

    \begin{subtable}{\textwidth}
        \centering
        \caption{\small{\textbf{Convex case.}
        $L_\ell$ and $(L_\nu, \nu)$ are (\Holder) smoothness parameters, $G$ is Lipschitz constant,
        $\smash{d_0\define\norm{\x^0 - \xs}}$, $\Delta(\cdot)$ is gradient noise function, $\Delta_D$ is maximum noise defined in~\pref{eq:variance-bound}. $\sigma_\star$ is value noise at $\xs$. \smash{$\gh^0$ and $\ellh^0$ are empirical values of $\nabla\ell(\x^0)$ and $\ell(\x^0)$}.
        The $\dag$ marks adaptivity to Hölder smoothness.
        Meanings of colors: {\color{myred}main convergence}, {\color{colorVar_1}variance}, {\color{colorSmall}optimum's noise}, and {\color{mygreen}ensemble error}. }}
        \renewcommand{\arraystretch}{1.4} 
        \resizebox{\textwidth}{!}{
        \begin{tabular}{c|cc|c|c} 
            \hline 
            
            \hline
            \multirow{2}{*}{\makecell{\textbf{Reference}}} & \multicolumn{2}{c|}{\textbf{Freeness}} & \multirow{2}{*}{\makecell{\textbf{Requirements}\\ \textbf{(Conditions)}}} & \multirow{2}{*}{\textbf{Convergence Rate of $\ell(\x^{\out}) - \ell(\xs)$}} \\[1mm] \cline{2-3} 
            & $L_\ell$ & $d_0$ & & \\ \hline
            \rule{0pt}{6mm}
            \makecell{~\citet{NeurIPS'19:UniXGrad}} & \cmark${}^\dag$ & \xmark & Domain bounded by $D$ & ${\color{myred}\frac{L_\nu D^{1+\nu}}{T^{\frac{1+3\nu}{2}}}} + {\color{colorVar_1}\frac{D\Delta_{D}}{\sqrt{T}}}\sbr{\log\frac{1}{\delta}}$ \\[1.5mm] \hline
            \rule{0pt}{6mm}
            \makecell{\citet{COLT'24:Kreisler}} & \cmark~ & \cxmark & \makecell{$\underset{(\ge\Delta(\cdot))}{\hat{\Delta}(\cdot)}$ ; $\underset{(\le d_0)}{r_\epsilon}$}  & $\sbr{{\color{myred}\min\bbr{ \frac{L_\ell d_0^2}{T^2}, \frac{Gd_0}{\sqrt{T}} }} + {\color{colorVar_1}\frac{d_0\Delta_{2d_0}}{\sqrt{T}} + \frac{d_0\hat{\Delta}_{2d_0}}{T}}  }\sbr{\log\frac{d_0}{r_\epsilon}}^2\sbr{\log\frac{1}{\delta}}^2 \sbr{ \log T \frac{\hat{\Delta}_{2d_0} + \min\{L_\ell d_0^2, G d_0\}}{\ell(\x_0) - \ell(\xs)} }^4$ \\[2mm] \hline
            \rule{0pt}{6mm}
            \makecell{\textbf{Ours [\pref{thm:con-search-gradient}]}} & \cmark~ & \cmark & \textemdash & ${\color{myred}\frac{\bar{L}_\ell \bar{d}_0^2}{T^2}}\sbr{\log\frac{\bar{d}_0}{d_\epsilon}}^2 + {\color{colorVar_1}\frac{\bar{d}_0\Delta_{2\bar{d}_0}}{\sqrt{T}}}\sbr{\log\frac{\bar{d}_0}{d_\epsilon}}^{\frac{1}{2}}\sbr{\log\frac{1}{\delta}} + {\color{mygreen}\textsc{Err}}\sbr{T\big/\sbr{\log \frac{T\norm{\gh^0}}{L_\epsilon d_\epsilon}}}$ \\[1mm] \hline
            \rule{0pt}{6mm}
            \makecell{\textbf{Ours [\pref{thm:con-search-value}]}} & \cmark${}^\dag$ & \cmark & \makecell{$\underset{(\le \ell(\xs))}{\ell_\epsilon^\star}$} & ${\color{myred}\tfrac{\bar{L}_\nu \bar{d}_0^{1+\nu}}{T^{\frac{1+3\nu}{2}}}}\sbr{\log\tfrac{\bar{d}_0}{d_\epsilon}}^{\frac{1+3\nu}{2}} + {\color{colorVar_1}\tfrac{\bar{d}_0\Delta_{2\bar{d}_0} }{\sqrt{T}}}\sbr{\log \tfrac{\bar{d}_0}{d_\epsilon}}^{\frac{1}{2}}\sbr{\log\tfrac{1}{\delta}} + {\color{colorSmall}\sigma_\star \sqrt{\tfrac{\log(1/\delta)}{T}}} + {\color{mygreen}\textsc{Err}}\sbr{T\big/\sbr{\log \frac{T(\ellh^0 -\ell_\epsilon^\star)}{L_\epsilon d_\epsilon}}}$ \\[1.5mm] 
            \hline 
            
            \hline
        \end{tabular}}
        \label{table:convex}
    \end{subtable}
\end{table*}

We aim to develop fully parameter-free methods by adopting the \emph{grid search} framework, a common approach in optimization practice. Existing grid search methods are partially parameter-free~\citep{ICML'24:Amit, khaled2024tuningfree}, as they rely on \emph{given} ranges of problem parameters to determine the search range.  However, in real-world scenarios, the ranges of problem parameters, especially their upper bounds, are often unavailable.

To this end, we propose a general and powerful grid search framework, along with a novel ``\emph{self-bounding}'' analysis technique to address the lack of theoretical guidance in determining the search ranges of problem parameters.
Our technique effectively derives computable search ranges for unknown problem parameters, while ensuring full parameter-freeness and maintaining competitive theoretical guarantees. 
Our methodology has been proven effective for both non-convex and convex optimization.

Our results significantly enrich the line of research on parameter-free stochastic optimization for both non-convex and convex cases, offering a series of new contributions:
\begin{itemize}[left=1pt, itemsep=0pt, topsep=1pt]
    \item For fully parameter-free stochastic optimization, we propose a methodology that enhances the general grid search framework with a novel \emph{self-bounding} analysis technique. This technique effectively derives the parameter search ranges without any prior knowledge of the problem parameters, as described in~\pref{subsec:self-bounding}.
    \item For {non-convex} and {smooth} stochastic optimization, we propose a \emph{fully} parameter-free method w.r.t. \emph{all} problem parameters, with convergence rate matching the optimal one up to logarithmic factors, as stated in~\pref{table:non-convex}.
    \item For {convex} and {smooth} stochastic optimization, we propose a \emph{fully} parameter-free method w.r.t. \emph{all} problem parameters, ensuring a near-optimal \mbox{\emph{accelerated}} convergence up to logarithmic factors. 
    For the more general \Holder smoothness, our method also demonstrates \emph{fully} parameter-freeness, ensuring \emph{universal} convergence, competitive with the state-of-the-art method of~\citet{COLT'24:Kreisler}. These results are summarized in~\pref{table:convex}.
    \item We provide a \emph{new} guarantee for model ensemble with an interpolated variance characterization, thereby enhancing grid search, as summarized in~\pref{table:ensemble}.
\end{itemize}
\textbf{Organization.}~~
The rest of the paper is organized as follows.
In \pref{sec:preliminary}, we introduce the preliminaries.
In \pref{sec:method}, we provide a general introduction to our grid search methodology.
In \pref{sec:non-convex} and \pref{sec:convex}, we provide the theoretical results for non-convex and convex optimization, respectively.
Then in~\pref{sec:exp}, we conduct experiments to verify the effectiveness of the proposed methods.
Finally, we conclude the paper in \pref{sec:conclusion}.
Due to page limits, most proofs are deferred to the appendix.


\section{Preliminaries}
\label{sec:preliminary}

In this section, we introduce optimization problem setups and definition of parameter-freeness used in this paper.

\textbf{Notations.}~~
We use $[N]$ to denote the index set $\{1,\dots,N\}$.
The $\ell_2$-norm is denoted by $\norm{\cdot}$.
Independently sampling from the oracle $\Q(\cdot)$ at $\x$ for $t$ times is denoted by $\Q_1(\x), \dots, \Q_t(\x)$.
We define $\log_+(x) \define 1 + \log(x)$.
We omit the $\log\log$ factor in the asymptotic notation $\O(\cdot)$. 

\subsection{Optimization Problem Setups}
\label{subsec:setups}
We focus on unconstrained stochastic optimization:
\begin{equation}
    \min_{\x\in\R^d}\ \ell(\x),
\end{equation}
where we denote the optimum by $\xs \in \argmin_{\x \in \R^d} \ell(\x)$. The algorithm starts from the initial point $\x^0$ and outputs a solution $\x^\textnormal{out}$. During the optimization process, we \mbox{assume} access to two stochastic oracles: one for gradients and one for function values, which is common in the machine learning and stochastic optimization literature.
\begin{myAssum}
    \label{assum:bounded-gradient-oracle}
    There exists a first-order oracle $\g(\cdot)$ such that: given any $\x \in \R^d$, it satisfies $\E[\g(\x)] = \nabla \ell(\x)$ and $\Pr[\norm{\g(\x) -\nabla \ell(\x)}\le \Delta(\x)\le \Delta_\ell]=1$, where $\Delta:\R^d\to\R_+$ is an \emph{unknown} function and $\Delta_\ell$ is an \emph{unknown} constant.
\end{myAssum}
\begin{myAssum}
    \label{assum:bounded-value-oracle}
    There exists a zeroth-order oracle $\ellt(\cdot)$ such that: given any $\x \in \R^d$, it satisfies $\E[\ellt(\x)] = \ell(\x)$ and $\Pr[|\ellt(\x)-\ell(\x)|\le \sigma_\ell]=1$, with an \emph{unknown} constant $\sigma_\ell$.
\end{myAssum}

Specifically, we interpret sampling from the oracle as drawing a differentiable function $\ell_{\xi}(\cdot)$ parameterized by the random variable $\xi$, where $\ellt(\cdot)=\ell_{\xi}(\cdot)$ and $\g(\cdot)=\nabla \ell_{\xi}(\cdot)$ represent the zeroth- and first-order oracles, respectively. This notation is omitted when the context makes it clear.

We then define the following maximum noise bound for a domain with center $\x^0$ and diameter $D$ as:
\begin{equation}
    \label{eq:variance-bound}
    \Delta_D \define \max\nolimits_{\norm{\x - \x^0}\le D} \Delta(\x).
\end{equation}
In the following, we provide the formal definitions of smoothness, Lipschitz continuity, and Hölder smoothness.
\begin{myDef}[Smoothness]
    \label{def:smoothness}
    The objective $\ell(\cdot)$ is $L_\ell$-smooth if $\norm{\nabla\ell(\x) - \nabla\ell(\y)} \le L_\ell \norm{\x - \y}$ for all $\x,\y \in \R^d$.
\end{myDef}
\begin{myDef}[Lipschitz Continuity]
    \label{def:lip-continuity}
    The objective $\ell(\cdot)$ is $G$-Lipschitz if $\abs{\ell(\x) - \ell(\y)} \le G \norm{\x - \y}$ for all $\x,\y \in \R^d$.
\end{myDef}
\begin{myDef}[Hölder Smoothness]
    \label{def:hoelder-smoothness}
    The objective $\ell(\cdot)$ is $(L_\nu, \nu)$-Hölder smooth with $L_\nu>0,\nu\in[0,1]$, if $\norm{\nabla\ell(\x) - \nabla\ell(\y)} \le L_\nu \norm{\x - \y}^\nu$ for all $\x,\y \in \R^d$.
\end{myDef}
We consider two setups: non-convex and convex objective.

For non-convex optimization, the objective $\ell(\cdot)$ is not necessarily convex but is $L_\ell$-smooth. 
We focus on the high-probability convergence rate of the squared gradient norm, $\norm{\nabla \ell(\x^\textnormal{out})}^2$, and the optimal rate of the tuned SGD relies on the functional gap between the initial point and the optimum, i.e. $F_\ell \define \ell(\x^0) - \ell(\xs)$, the maximum gradient noise  $\Delta_\ell$, and the smoothness parameter $L_\ell$.

For convex optimization, we focus on the high-probability convergence rate of the sub-optimality gap $\ell(\x^\textnormal{out}) - \ell(\xs)$.
For accelerated convergence, we assume the objective is $L_\ell$-smooth.
For the more challenging universal convergence, we assume that the objective function satisfies one of the following cases: \rom{1} either $L_\ell$-smooth or $G$-Lipschitz; \rom{2} $(L_\nu, \nu)$-Hölder smooth.
The optimal rate relies on the distance between the initial point and the optimum \smash{$d_0 \define \|\x^0 - \xs\|$}, the maximum gradient noise  $\Delta_\ell$, and the smoothness parameters $L_\ell$ or $(L_\nu, \nu)$.

\subsection{Definition of Parameter-Freeness}
\label{subsec:def-parameter-free}
In this part, we introduce partial and full parameter-freeness in order to better distinguish between different levels of parameter dependence and to facilitate a clear comparison with previous works.
{First, we assume the algorithm has access to the oracle query budget $T$ and the confidence level $\delta$.}
Regarding the requirement for prior knowledge of problem parameters, such as $\Delta_\ell, \sigma_\ell, L_\ell, F_\ell, d_0$ defined in~\pref{subsec:setups}, we propose the following definitions.

\begin{myDef}[Partial\ /\ Full Parameter-Freeness]
    \label{def:parameter-free}
    An optimization method is \emph{parameter-free} w.r.t. some problem parameter $X$ if it does not require $X$ as an input. Formally,
    \begin{itemize}[left=1pt, itemsep=0pt, topsep=0pt]
    \item \textbf{Partial parameter-freeness:} It is agnostic to $X$, but some algorithmic inputs must satisfy an \emph{unverifiable} \mbox{condition} related to $X$. For example, the input $X^\prime$ must satisfy that $X^\prime \leq X$, but there is no guarantee that this condition holds.
    \item \textbf{Full parameter-freeness:} Its algorithmic inputs do not need to satisfy any \emph{unverifiable} conditions related to $X$.
    \end{itemize}    
\end{myDef}

It is important to note that parameter-free does not imply freeness from \emph{all} parameters; rather, it specifically refers to the absence of key problem parameters related to the problem's properties, such as $\Delta_\ell, \sigma_\ell, L_\ell, F_\ell, d_0$ defined in \pref{subsec:setups}. 
In the rest of the paper, we abbreviate the above definitions as ``partially\ /\ fully free to $X$'' for convenience.

\section{Our Grid Search Framework}
\label{sec:method}
In this section, we introduce our grid search framework, \textsc{Grasp} (\underline{GR}id-se\underline{A}rch with \underline{S}elf-bounding for \underline{P}arameter-free Optimization), including three key components: self-bounding analysis, budget allocation, and model ensemble.

\textbf{Framework Overview.}~~
Given an algorithm $\A$ with parameters to be tuned, such as the step size, and a total oracle budget $T$, the grid search framework proceeds as follows:
\begin{itemize}[left=10pt, itemsep=-2pt, topsep=-3pt]
    \item[(1)] Identify a suitable range for each algorithmic parameter, where our self-bounding technique assists in doing this.
    \item[(2)] Discretize the search range and allocate the budget; for each value in the discretized set, run an instance of the algorithm $\A$ using the corresponding parameter values.
    \item[(3)] Select the best output from these algorithm runs.
\end{itemize}

Then, in~\pref{subsec:example-SCO}, we use deterministic convex and smooth optimization as an example to illustrate the entire pipeline of the grid search methodology.

\subsection{General Idea of Self-Bounding Analysis}
\label{subsec:self-bounding}
The grid search framework is commonly used in optimization practice; however, it typically relies on manually specified ranges for algorithmic inputs, lacking theoretical guidance for selecting these ranges.
To address this, we propose the ``self-bounding'' analysis technique as a simple yet effective solution. We use convex and smooth optimization as an illustrative example and then present the general idea.

\textbf{An Intuitive Example.}~~Consider the convex and smooth optimization with \emph{deterministic gradients}, where the optimal convergence is $\O(L_\ell d_0^2 / T^2)$, with $d_0$ being the initial distance to the optimum. 
Since $d_0$ could be infinitely large, it is difficult to determine a suitable search range for it.
Our key insight is that, \emph{this $d_0$ is naturally bounded to ensure the accelerated rate is non-vacuous}. Specifically, the target rate has a \emph{quadratic} dependence on $d_0$, so if $d_0$ is excessively large, the convergence rate will be worse than a trivial bound, e.g., $\ell(\x^0) - \ell(\xs) \le \norm{\nabla \ell(\x^0)} d_0$.
To this end, we use this trivial bound as a benchmark to derive a reasonable upper bound for $d_0$, that is, by solving the inequality:
\begin{equation}
    \label{eq:self-bounding-example}
    \norm{\nabla \ell(\x^0)} d_0 \ge \frac{\max\{L_\ell, L_\epsilon\} d_0^2}{T^2} \ge \frac{L_\ell d_0^2}{T^2} ,
\end{equation}
thereby leading to a computable and effective upper bound of $d_0 \le \norm{\nabla \ell(\x^0)}T^2/L_\epsilon$ with arbitrary $L_\epsilon>0$.

This example illustrates the main idea of our self-bounding analysis: we can determine an \emph{effective} search range for unknown problem parameters, given appropriate connections between the target and benchmark rates.
Below, we present a more rigorous formulation of this idea.

\textbf{Self-Bounding Analysis.}~~
Consider the target rate of:
\begin{equation}
    \label{eq:old-target}
    \eps^{\textsc{tar}}\sbr{T,p^1,\ldots,p^m},
\end{equation}
where \smash{$p^1,\ldots,p^m\in\R_+$} are unknown problem parameters, and $\eps^{\textsc{tar}}$ is a non-decreasing function of these parameters. 
For convex and smooth optimization, the target rate is $\eps^{\textsc{tar}}(T,L_\ell,d_0)={L_\ell d_0^2}/{T^2}$, which depends on the smoothness parameter $L_\ell$ and the initial distance $d_0$.

For each parameter $p^i$, we require a \emph{user-specified} input $p_\epsilon^i>0$, which can generally be very small.
Next, we define $\bar{p}^i \define \max\{p^i, p_\epsilon^i\}$, and shift our target to:
\begin{equation}
    \label{eq:self-bounding-new-target}
    \eps^{\textsc{tar}}\sbr{T,\bar{p}^1,\ldots,\bar{p}^m} \ge \text{\pref{eq:old-target}}.
\end{equation}
In our convex and smooth optimization example, our target rate becomes \smash{$\O(\bar{L}_\ell \bar{d}_0^2/T^2)$}, where $\bar{L}_\ell\define \max\{L_\ell, L_\epsilon\}$ and $\bar{d}_0 \define \max\{ d_0, d_\epsilon \}$.
Note that the user-specified $L_\epsilon$ and $d_\epsilon$ do \emph{not} necessarily satisfy any conditions w.r.t. $L_\ell$ or $d_0$, which does not conflict with the full parameter-freeness.

Subsequently, we find a point $\x^{\textsc{bm}}$ with a trivial benchmark convergence rate of $\eps^{\textsc{bm}}$.
In our example, simply choosing $\x^{\textsc{bm}}=\x^0$ yields $\eps^{\textsc{bm}} \define \norm{\nabla \ell(\x^0)} d_0$.

Next we can derive an upper bound for grid search. 
For an unknown parameter $p^i$, solving the following inequality:
\begin{equation*}
    \resizebox{\columnwidth}{!}{$
    p^i_{\max} = \argmax_{p^i} \bbr{\eps^{\textsc{bm}} \ge \eps^{\textsc{tar}}(T,\bar{p}^1,\ldots,p^i,\ldots,\bar{p}^m)}
    $}
\end{equation*}
and we set $p^i_{\max}$ as an effective \emph{upper bound} for $p^i$. 
This upper bound is not the theoretically largest, but if $p^i$ exceeded it, the target rate would become worse than that of the benchmark $\eps^{\textsc{bm}}$, making the final rate vacuous.
Therefore, to achieve the target rate in~\pref{eq:self-bounding-new-target}, it suffices to focus on the range of $[p_\epsilon^i, \max\{p_\epsilon^i, p_{\max}^i\}]$ for each $\bar{p}^i$.

Finally, we ensemble all candidates within the grids, including the benchmark $\x^{\textsc{bm}}$, as in~\citet{ICML'24:Amit,khaled2024tuningfree}, to ensure the following convergence:
\begin{equation}
    \label{eq:self-bounding-result}
\eps^{\textsc{tar}}\sbr{T,\bar{p}^1,\ldots,\bar{p}^m} \poly\log\sbr{\frac{p^1_{\max}}{p_\epsilon^1},\ldots,\frac{p^m_{\max}}{p_\epsilon^m}}.
\end{equation}
In general, our result~\eqref{eq:self-bounding-result} will eventually suffer a logarithmic factor on $1/p^i_\epsilon$ for each $i \in [m]$. 
Thus, setting a small user-specified $p^i_\epsilon$ will \emph{not} significantly affect the final rate.

\subsection{An Initialization for Deterministic Optimization}
\label{subsec:example-SCO}
In this part, we initialize our framework within the deterministic convex smooth optimization setting as a clearer demonstration. 
As shown above, for the target rate of $\O(\bar{L}_\ell \bar{d}_0^2 / T^2)$, the self-bounding analysis provides a search range of $[d_\epsilon, d_{\max}]$ for $\bar{d}_0$, where $d_{\max} \define \max\{d_\epsilon, \|\nabla \ell(\x^0)\|T^2/L_\epsilon\}$ with arbitrary $L_\epsilon,d_\epsilon>0$.

\textbf{Discretization and Allocation.}~~
We discretize the range $[d_\epsilon, d_{\max}]$ using a geometric sequence, i.e., let $D_i=d_\epsilon 2^i$ for $i\in[N]$, where $N=\lceil \log_2 ({d_{\max}}/{d_\epsilon})\rceil$. 
We allocate an oracle budget of $T_i$ for each \UniXGrad~\citep{NeurIPS'19:UniXGrad}, a base algorithm for convex smooth optimization, fed with $D_i$.
In the case where there exists $\is\in[N]$ such that $D_{\is}/2\le \bar{d}_0 \le D_{\is}$, the $\is$-th \UniXGrad ensures:
\begin{equation}
    \label{eq:allocation-target}
    \O\sbr{\frac{\bar{L}_\ell D_{\is}^2}{T_{\is}^2}} \le \O\sbr{ \frac{\bar{L}_\ell \bar{d}_0^2}{T_{\is}^2} }.
\end{equation}
Then a simple \emph{uniform} allocation of $T_i=\lfloor \frac{T}{N} \rfloor$ yields:
\begin{equation}
    \label{eq:allocation-target-uniform-result}
    \O\sbr{ \frac{\bar{L}_\ell \bar{d}_0^2}{T_{\is}^2} } = \O\sbr{ \frac{\bar{L}_\ell \bar{d}_0^2}{T^2}\sbr{\log \frac{d_{\max}}{d_\epsilon}}^2 }.
\end{equation}
Moreover, the above convergence rate can be further improved via a \mbox{\emph{non-uniform}} budget allocation. 
Specifically, with $T_i=\lfloor {T}/{(i(1+\ln N))} \rfloor$ for $i\in[N]$, we obtain:
\begin{equation}
    \label{eq:allocation-target-non-uniform-result}
    \O\sbr{ \frac{\bar{L}_\ell \bar{d}_0^2}{T_{\is}^2} } = \O\sbr{ \frac{\bar{L}_\ell \bar{d}_0^2}{T^2}\sbr{\log \frac{\bar{d}_0}{d_\epsilon}}^2 },
\end{equation}
which replaces the $d_{\max}$ in \pref{eq:allocation-target-uniform-result} by $\bar{d}_0$.
And the dependency on $\log (d_0/d_\epsilon)$ matches~\citet[Theorem 1]{COLT'24:Kreisler}, but without the additional logarithmic factors therein.

\textbf{Ensemble.}~~
In the deterministic case, we simply select the candidate with the minimum function value. 
While in the stochastic setting with access to a noisy function value oracle, we need to allocate more oracle budget to each candidate to obtain an accurate estimate of the function value.
We then select $\x^\textnormal{out}$ with the smallest estimated function value, ensuring that, for any candidate $\x^i$ where $i\in[N]\cup\{0\}$:
\begin{equation}
    \ell(\x^\textnormal{out}) - \ell(\xs) \le \O\sbr{\ell(\x^i) - \ell(\xs)} + \textsc{Err},
\end{equation}
where the ensemble error $\textsc{Err}$ denotes the maximum estimation error among all candidates.

Besides, we emphasize that providing a sharper ensemble error is important, as it depends on the zeroth-order noise $\sigma_\ell$ given in \pref{assum:bounded-value-oracle} and thus cannot be simply absorbed into gradient variance terms.
To this end, we provide a more refined depiction of the zeroth-order stochastic variance, thereby obtaining a sharper, problem-dependent ensemble error in favorable regimes such as interpolation.
More details will be provided in~\pref{subsec:selection}.

To conclude, we present~\pref{thm:con-search-gradient-deterministic} for deterministic convex and smooth optimization, with a proof sketch. Notably, it is \emph{fully parameter-free} without any knowledge of $L_\ell$ and $d_0$.

\begin{algorithm}[!t]
    \caption{Deterministic Convex and Smooth OPT}
    \label{alg:grid-search-deterministic}
    \begin{algorithmic}[1]
    \REQUIRE Oracle budget $T$, initial point $\x^0$, $d_\epsilon > 0$, $L_\epsilon>0$.
    \STATE \textbf{Initialization:} Set $N = \lceil \log_2 ({d_{\max}}/{d_\epsilon}) \rceil$ with search upper bound $d_{\max} \define \max\{ d_\epsilon, {\norm{\nabla \ell(\x^0)} T^2}/{L_\epsilon}\}$
    \FOR{$i=1,2,\dots,N$}
        \STATE Run \UniXGrad~(see~\pref{lemma:con-base-bounded}) with initial point $\x^0$, domain diameter $D_i=d_\epsilon 2^i$, oracle budget $T_i=\lfloor {T}/{(i(1+\ln N))} \rfloor$, then receive the output $\x^i$
    \ENDFOR
    \ENSURE $\x^\out = \x^{\is}$ with $\is = \argmin_{0\le i\le N} \ell(\x^i)$.
    \end{algorithmic}
\end{algorithm}
\begin{myThm}
    \label{thm:con-search-gradient-deterministic}
    For the deterministic convex optimization, and assume the objective $\ell(\x)$ is $L_\ell$-smooth. With any user-specified $L_\epsilon,d_\epsilon>0$, Algorithm~\ref{alg:grid-search-deterministic} enjoys:
    \begin{equation*}
        \ell(\x^\textnormal{out}) - \ell(\xs) \le \O\sbr{ \frac{\bar{L}_\ell \bar{d}_0^2}{T^2}\sbr{\log_+\frac{\bar{d}_0}{d_\epsilon}}^2 },
    \end{equation*}
    where $\bar{L}_\ell \define \max\{L_\ell, L_\epsilon\}$, $\bar{d}_0 \define \max\{d_0, d_\epsilon\}$, $d_0\define\norm{\x^0 - \xs}$, and we omit the double logarithmic factor $\log\log (1/L_\epsilon)$ in notation $\O(\cdot)$. 
\end{myThm}

\begin{myRemark}
    \pref{thm:con-search-gradient-deterministic} is \emph{fully} parameter-free to both $L_\ell$ and $d_0$. 
    We are aware of two related works for the same setting. 
    The first one is~\citet[Theorem 1]{COLT'24:Kreisler}, which, besides requiring prior knowledge of $r_\epsilon \le d_0$, has an additional logarithmic factor of $\log _{+}^4 \big(1+T \min \{L_\ell d_0^2, G d_0\}/F_\ell \big)$. The other one is~\citet[Corollary 3]{li2025simple}, an optimal and fully parameter-free result, but its extension to the stochastic setting remains unclear.
\end{myRemark}
\begin{proof}[Proof Sketch of~\pref{thm:con-search-gradient-deterministic}]
Since $\x^{\textnormal{out}}$ is the best candidate,
\begin{equation*}
\ell(\x^\textnormal{out}) - \ell(\xs) \le \ell(\x^i) - \ell(\xs),
\end{equation*}
for all $i\in[N]\cup\{0\}$. We define the search upper bound $d_{\max} \define \max\{ d_\epsilon, {\norm{\nabla \ell(\x^0)} T^2}/{L_\epsilon}\}$, and perform the following case-by-case study for $\bar{d}_0\define \max\{d_0, d_\epsilon\}$.

\textbf{Case of $\bar{d_0} > d_{\max}$.}~~Then $\frac{\bar{L}_\ell \bar{d}_0^2}{T^2} > \norm{\nabla \ell(\x^0)}d_0$ implies:
\begin{align*}
    \ell(\x^\out) - \ell(\xs) &\le \ell(\x^0) - \ell(\xs) \\
    &\le \norm{ \nabla \ell(\x^0) } d_0 \le \O\sbr{\frac{\bar{L}_\ell \bar{d}_0^2}{T^2} }.
\end{align*}
\textbf{Case of $\bar{d}_0\in [d_\epsilon, d_{\max}]$.}~~Let $\is = \lceil \log_2 ({\bar{d}_0}/{d_\epsilon})\rceil \in [N]$, its diameter $D_{\is} = d_\epsilon 2^{\is}$ that $D_{\is}/2 \le \bar{d}_0 \le D_{\is}$, and allocated oracle budget is $T_{\is} = \lfloor {T}/{(\is(1+\ln N))} \rfloor$. Applying~\pref{lemma:con-base-bounded} with $L_\nu=L_\ell,\nu=1$, $\delta\to 0_+$, $\Delta_\ell=0$:
\begin{align*}
    \ell(\x^\out) - \ell(\xs) &\le \ell(\x^{\is}) - \ell(\xs) \le \O\sbr{ \frac{\bar{L}_\ell D_{\is}^2}{T_{\is}^2} } \\
    &\le \O\sbr{ \frac{ \bar{L}_\ell \bar{d}_0^2}{T^2}\sbr{\log_+\frac{\bar{d}_0}{d_\epsilon}}^2 },
\end{align*}
where $\O(\cdot)$ omits the $\log\log (1/L_\epsilon)$ factor. Combining the above two cases completes the proof.
\end{proof}

\section{Parameter-Free Non-Convex Optimization}
\label{sec:non-convex}
In this section, we focus on non-convex optimization, where the objective $\ell(\x)$ is $L_\ell$-smooth.
In this setup, the best-known rate is achieved by Stochastic Gradient Descent (SGD)~\citep{ghadimi2013stochastic}, for which the fixed step size must be carefully tuned with respect to the smoothness parameter $L_\ell$, the initial sub-optimality gap $F_\ell$, and the gradient variance bound $\Delta_\ell$, as defined in~\pref{subsec:setups}.

Building on the general grid search methodology and our self-bounding analysis proposed in \pref{sec:method}, along with an efficient ensemble method from~\citet{ICML'24:Amit}, we propose a novel \emph{fully parameter-free} method, named  \textsc{Grasp-NC} for \underline{N}on-\underline{C}onvex optimization, as shown in~\pref{alg:non-convex}.
Our result matches the optimal rate of \emph{tuned} SGD~\citep{ghadimi2013stochastic}, up to logarithmic factors, while does \emph{not} require any information about $L_\ell$, $F_\ell$ and $\Delta_\ell$.
Specifically, in \pref{line:nc-sampling}, we sample $\g(\x^0)$ for multiple times to obtain an accurate gradient estimation $\gh^0$.
In Lines~\ref{line:nc-self-bounding-1}-\ref{line:nc-self-bounding-2}, we use \smash{$\|\gh^0\|$} to construct effective boundaries for the problem parameters of $L_\ell, F_\ell$ and $\Delta_\ell$.
In \pref{line:nc-base}, we run SGD with a constant step size $\eta_i=\eta_{\min} 2^i$ in the $i$-th grid.
Then in \pref{line:nc-ensemble}, we sample each candidate for multiple times to obtain accurate estimations.
Finally, we select the point with minimum gradient norm as the output.

Below, we provide the guarantee, with proof in~\pref{appendix:proof-non-convex}.

\begin{algorithm}[!t]
    \caption{\textsc{Grasp-NC}}
    \label{alg:non-convex}
    \begin{algorithmic}[1]
    \REQUIRE Oracle budget $T$, initial point $\x^0$, $L_\epsilon>0,F_\epsilon>0$, and confidence level $\delta\in(0,1/3)$.
    \STATE Independently sample $\g(\cdot)$ at $\x^0$ for $\frac{T}{4}$ times, calculate average ${\gh^0 = \frac{4}{T}\sum_{t=1}^{T/4}\g_t(\x^0)}$ \label{line:nc-sampling}
    \STATE $L_{\max}\define\frac{\norm{\gh^0}^2 T}{F_\epsilon}$, $F_{\max}\define\frac{\norm{\gh^0}^2 T}{L_\epsilon}$, $\Delta_{\max}^2\define\frac{\norm{\gh^0}^2 T}{\log \frac{1}{\delta}}$ \label{line:nc-self-bounding-1}
    \STATE Set the search range for step size: $\eta_{\max} \define \frac{1}{2L_\epsilon}$, and
    \begin{equation*}
        \eta_{\min} \define \min\bbr{ \tfrac{1}{2\max\{L_\epsilon, L_{\max}\}}, \sqrt{ \tfrac{2 F_\epsilon}{\max\{L_\epsilon, L_{\max}\} \Delta_{\max}^2 T} } }
    \end{equation*} \vspace{-1em} \label{line:nc-self-bounding-2}
    \STATE $N = \lceil \log_2 \frac{\eta_{\max}}{\eta_{\min}} \rceil$, $K=\lceil\log_2\frac{1}{\delta}\rceil$, candidate set $\S = \emptyset$
    \FOR{$i=1,2,\dots,N$}
        \STATE Run SGD~(see~\pref{lemma:SGD}) with initial point $\x_1=\x^0$, constant step size $\eta_i=\eta_{\min} 2^i$, oracle budget $\lfloor \frac{T}{2N} \rfloor$, then uniformly sample $K$ points from the algorithm's trajectory and add to the candidate set $\S$ \label{line:nc-base}
    \ENDFOR
    \STATE Order points in $\S$ as $\x^1,\dots,\x^{KN}$
    \FOR{$i=1,2,\dots,KN$}
        \STATE Independently sample $\g(\cdot)$ at $\x^i$ for $\frac{T}{4KN}$ times, and calculate average $\gh^i = \frac{4KN}{T}\sum_{t=1}^{T/(4KN)}\g (\x^i)$ \label{line:nc-ensemble}
    \ENDFOR
    \ENSURE $\x^\out = \x^{\is}$ with $\is = \argmin_{0\le i\le KN} \norm{\gh^i}$. \label{line:nc-output}
    \end{algorithmic}
\end{algorithm}

\begin{myThm}
    \label{thm:non-convex}
    Under Assumption~\ref{assum:bounded-gradient-oracle}, and assume the objective $\ell(\x)$ is $L_\ell$-smooth. Algorithm~\ref{alg:non-convex} ensures that, for any $\delta\in(0,1/3)$, with probability at least $1-3\delta$:
    \begin{align*}
        \norm{\nabla\ell(\x^\textnormal{out})}^2 &\le \O\Bigg(\sqrt{\frac{\bar{L}_\ell\bar{F}_\ell {\Delta}_\ell^2}{T} \log_+ \frac{T \norm{\gh^0}}{L_\epsilon F_\epsilon} }  \\
        &\qquad\quad {} + \frac{\bar{L}_\ell\bar{F}_\ell + \Delta_{\ell}^2(\log\frac{1}{\delta})^2}{T}\log_+ \frac{T \norm{\gh^0}}{L_\epsilon F_\epsilon} \Bigg),
    \end{align*}
    where $\bar{L}_\ell \define \max\{L_\ell, L_\epsilon\}$, $\bar{F}_\ell\define \max\{\ell(\x^0) - \ell(\x^\star), F_\epsilon\}$.
\end{myThm}
We provide a detailed comparison of our result with those of tuned SGD and \citet{ICML'24:Amit} in~\pref{table:non-convex}.
\begin{myRemark}
In~\pref{alg:non-convex}, the algorithmic inputs $L_\epsilon, F_\epsilon$ do not need to satisfy any conditions related to $L_\ell$ and $F_\ell$, validating that our method is \emph{fully} parameter-free.
However, compared with tuned SGD, in~\pref{thm:non-convex}, the dependency on $L_\ell$ becomes $\max\{L_\ell, L_\epsilon\}$. 
This is acceptable because the user can choose a small $L_\epsilon$, while only affecting the convergence rate by a $\log(1/L_\epsilon)$ factor.
\end{myRemark}

\begin{myRemark}
In \pref{line:nc-sampling} of \pref{alg:non-convex}, the initial sampling process at $\x^0$ consumes \smash{$\frac{T}{4}$} oracle calls. 
We emphasize that this is mainly for theoretical analysis to ensure a tight convergence rate. 
In practice, sampling less may not significantly affect performance, as shown in~\pref{sec:exp}.
Intuitively, in the proof, only if $\x^0$ becomes the best candidate, the rate introduces an estimation error for $\norm{\gh^0}$, necessitating a sufficiently large sample size. However, in practice, the probability of ``$\x^0$ is the best candidate'' is small, so reducing the initial sampling may not cause a substantial effect.
\end{myRemark}

\section{Parameter-Free Convex Optimization}
\label{sec:convex}

In this section, we investigate convex optimization, where the objective $\ell(\x)$ is convex with an unknown smoothness level, as introduced in~\pref{subsec:setups}.
Moreover, we aim to achieve ``universality''~\citep{nesterov2015universal}, i.e., adaptation to an unknown level of (\Holder) smoothness while maintaining the optimal convergence.

\textbf{Base Algorithm.}~~
In the stochastic setting, the \UniXGrad algorithm~\citep{NeurIPS'19:UniXGrad}, works in bounded domain with diameter $D$, can ensure a universally optimal convergence rate between the separate cases of smoothness and non-smoothness. 
In~\pref{lemma:con-base-bounded}, we prove that \UniXGrad can achieve the stronger universality to \Holder smoothness:
\begin{equation*}
    \O\sbr{ \frac{L_\nu D^{1+\nu}}{T^{\frac{1+3\nu}{2}}} + \frac{\Delta_D D\log(1/\delta)}{\sqrt{T}}},
\end{equation*}
where $\Delta_D$ is the maximum noise bound defined in~\pref{eq:variance-bound}. It matches the optimal rate in the deterministic setting~\citep{nesterov2015universal}, by recovering the optimal rates in both $L_\ell$-smooth case and $G$-Lipschitz case, respectively.

\textbf{Algorithms and Convergence Rates.}~~
\label{subsec:convex}
In this part, we propose \textsc{Grasp-C}, named \textsc{Grasp} for \underline{C}onvex optimization, in~\pref{alg:convex}. 
Specifically, \pref{line:cvx-d-bar} offers two options for the searching upper bound of the initial distance $d_0$:
\begin{itemize}[left=1pt, itemsep=0mm, topsep=0pt]
    \item \textsc{Option-I}~\eqref{eq:grid-search-ub-gradient} for \emph{acceleration-only} in the smooth case;
    \item \textsc{Option-II}~\eqref{eq:grid-search-ub-value} for \emph{universality} to \Holder smoothness, but requires a lower bound of the objective value.
\end{itemize}

\begin{algorithm}[!t]
    \caption{\textsc{Grasp-C}}
    \label{alg:convex}
    \begin{algorithmic}[1]
    \REQUIRE Oracle budget $T$, initial point $\x^0$, $d_\epsilon > 0$, $L_\epsilon>0$, and $\ell_\epsilon^\star$ if choose $\textsc{Option-II}$ in~\pref{eq:grid-search-ub-value}.
    \STATE Independently sample $\g(\cdot)$ at $\x^0$ for $\frac{T}{8}$ times, calculate average $\gh^0 = \frac{8}{T}\sum_{t=1}^{T/8}\g_t(\x^0)$
    \STATE Independently sample $\ellt(\cdot)$ at $\x^0$ for $\frac{T}{8}$ times, calculate average $\ellh^0 = \frac{8}{T}\sum_{t=1}^{T/8}\ellt_t(\x^0)$
    \STATE Set $N = \lceil \log_2 \frac{d_{\max}}{d_\epsilon} \rceil$ with grid search upper bound $d_{\max}$ using one of the following two options:
    \begin{align}
        \textsc{Option-I}: \ d_{\max} &\define \max\{ d_\epsilon, \tfrac{\norm{\gh^0} T^2}{L_\epsilon}\} \label{eq:grid-search-ub-gradient} \\
        \textsc{Option-II}: \ d_{\max} &\define \max\{1, d_\epsilon, \tfrac{(\ellh^0-\ell_\epsilon^\star) T^2}{L_\epsilon}\} \label{eq:grid-search-ub-value}
    \end{align} \vspace{-1em} \label{line:cvx-d-bar}
    \FOR{$i=1,2,\dots,N$}
        \STATE Run \UniXGrad~(see~\pref{lemma:con-base-bounded}) with oracle budget $\lfloor \frac{T}{2i(1+\ln N)} \rfloor$, domain diameter $d_\epsilon 2^i$, initial point $\x^0$, and receive the output $\x^i$ from the algorithm
        \STATE Independently sample $\ellt(\x^i)$ for $\frac{T}{4N}$ times, calculate average $\ellh^i = \frac{4N}{T}\sum_{t=1}^{T/(4N)}\ellt_t (\x^i)$
    \ENDFOR
    \ENSURE $\x^\out = \x^{\is}$ with $\is = \argmin_{0\le i\le N} \ellh^i$.
    \end{algorithmic}
\end{algorithm}

In~\pref{thm:con-search-gradient} below, we prove that our \emph{fully parameter-free} algorithm with \mbox{\textsc{Option-I}} can achieve near-optimal convergence rate up to logarithmic factors, with an additive ensemble error term.
The proof is in~\pref{appendix:proof-con-search-gradient}.
\begin{myThm}
    \label{thm:con-search-gradient}
    Under Assumptions~\ref{assum:bounded-gradient-oracle} and~\ref{assum:bounded-value-oracle}, and assume the objective $\ell(\x)$ is $L_\ell$-smooth.
    Algorithm~\ref{alg:convex} using \textsc{Option-I} in Eq.~\eqref{eq:grid-search-ub-gradient} ensures that, for any $\delta\in(0,1/2)$, with probability at least $1-2\delta$, $\ell(\x^\textnormal{out}) - \ell(\xs)$ is bounded by:
    \begin{align*}
        & \O \Bigg( \frac{\bar{L}_\ell \bar{d}_0^2}{T^2} \sbr{{\log\frac{\bar{d}_0}{d_\epsilon}}}^2 + \frac{\Delta_{2\bar{d}_0}\bar{d}_0}{\sqrt{T}}\sbr{\log\frac{\bar{d}_0}{d_\epsilon}}^{\frac{1}{2}}\Big({\log\frac{1}{\delta}}\Big) \\
        &\qquad + \textsc{Err}_N\sbr{{\frac{T}{N}}} \Bigg),
    \end{align*}
    where \smash{$\bar{L}_\ell \define \max\{L_\ell, L_\epsilon\}$ and $\bar{d}_0 \define \max\{d_0, d_\epsilon\}$} with any $L_\epsilon,d_\epsilon>0$. $\Delta_{2\bar{d}_0}$ is the maximum noise defined in Eq.~\eqref{eq:variance-bound}. $N=\O (\log (T\norm{\gh^0}/(L_\epsilon d_\epsilon)))$, and $\textsc{Err}_N(\cdot)$ is an ensemble error formally defined in Section~\ref{subsec:selection}. 
\end{myThm}

Compared to guarantee of \UniXGrad, our~\pref{thm:con-search-gradient} includes an additional ensemble error $\textsc{Err}$ term. 
We will provide its analysis and discussion in the following.

Meanwhile, by setting \textsc{Option-II} in~\pref{eq:grid-search-ub-value}, our \pref{alg:convex} can adapt to \Holder smoothness while maintaining the \emph{fully} parameter-freeness.
However, it requires a lower bound of the objective value $\ell_\epsilon^\star \le \ell(\xs)$ to construct the benchmark convergence rate.
We present the convergence rate in the following theorem, with proof in~\pref{appendix:proof-con-search-value}.
\begin{myThm}
    \label{thm:con-search-value}
    Under Assumptions~\ref{assum:bounded-gradient-oracle} and~\ref{assum:bounded-value-oracle}, and assume the objective $\ell(\x)$ is $(L_\nu,\nu)$-\Holder smooth and $\ell(\xs) \ge \ell_\epsilon^\star$.
    Algorithm~\ref{alg:convex} using \textsc{Option-II} in Eq.~\eqref{eq:grid-search-ub-value} ensures that, for any $\delta\in(0,1/2)$, with probability at least $1-2\delta$, $\ell(\x^\textnormal{out}) - \ell(\xs)$ is bounded by:
    \begin{align*}
        & \O \Bigg( \frac{\bar{L}_\nu \bar{d}_0^{1+\nu}}{T^{\frac{1+3\nu}{2}}}\sbr{\log\frac{\bar{d}_0}{d_\epsilon}}^{\frac{1+3\nu}{2}} + \frac{\Delta_{2\bar{d}_0}\bar{d}_0 }{\sqrt{T}}\sbr{\log \frac{\bar{d}_0}{d_\epsilon}}^{\frac{1}{2}}\Big({\log\frac{1}{\delta}}\Big) \\
        &\qquad + \frac{\sigma_\star \sqrt{\log(1/\delta)}}{\sqrt{T}} + \textsc{Err}_N\sbr{\frac{T}{N}} \Bigg),
    \end{align*}
    \mbox{where $\smash{\bar{L}_\nu \define \max\{L_\nu, L_\epsilon\}}$ and $\smash{\bar{d}_0 \define \max\{d_0, d_\epsilon\}}$ with any} $L_\epsilon,d_\epsilon>0$, $\Delta_{2\bar{d}_0}$ is the maximum noise defined in Eq.~\eqref{eq:variance-bound}. $\sigma_\star$ is value noise at $\xs$, i.e. $\smash{\Pr[|\ellt(\xs) - \ell(\xs)| \le \sigma_\star] = 1}$. $N=\O (\log (T(\ellh^0 - \ell_\epsilon^\star)/(L_\epsilon d_\epsilon)))$, and $\textsc{Err}_N(\cdot)$ is an ensemble error formally defined in Section~\ref{subsec:selection}.
\end{myThm}

Notably, the assumption of $\ell_\epsilon^\star \le \ell(\xs)$ is mild and naturally satisfiable in many cases. 
For example, for non-negative losses, such as the widely used squared losses, cross-entropy losses, hinge losses, etc., we can simply set $\ell_\epsilon^\star = 0$.
We also provide~\pref{thm:con-search-T/i2} in appendix to eliminate the need for $\ell_\epsilon^\star$ while still maintaining universality and full freeness, at the cost of a worse ensemble error of $\mathcal{O}(\textsc{Err}_{\sqrt{T}}(\sqrt{T}))$.

Additionally, \pref{thm:con-search-value} introduces an extra term of the zeroth-order noise at $\xs$, i.e., $\sigma_\star /\sqrt{T}$. This can be considered small, especially in cases with over-parameterization where $\sigma_\star \approx 0$~\citep{ICML'18:Ma,NeurIPS'23:Liu}.


\subsection{Sharper Ensemble under Interpolation}
\label{subsec:selection}
We denote by $\textsc{Err}_{N}(M)$ the ensemble error, satisfying that:
\begin{equation*}
    \ell(\x^\out) - \ell(\xs) \le \O( \min_{i\in[N]} \ell(\x_i) - \ell(\xs) + \textsc{Err}_{N}(M)),
\end{equation*}
where $N$ candidates are sampled $M$ times, respectively.
Comparing function values for ensemble introduces the zeroth-order variance $\sigma_\ell$ to the ensemble error.
Specifically, as formally stated in \pref{lemma:ensemble}, comparison of function values will lead to a zeroth-order variance $\sigma_\ell/\sqrt{M}$ to the ensemble error, where $M$ is the number of samples. 

\begin{table}[!t]
    \centering
    \caption{\small{Comparison of ensemble guarantees, i.e.~\citet[Lemma 9]{ICML'24:Amit} and our~\pref{lemma:selection-variance}, given stochastic function value oracle (\pref{assum:bounded-value-oracle}).
    $\textsc{Err}_{N}(M)$ is the ensemble error.
    `\textemdash' denotes no additional condition. \smash{$\Ot(\cdot)$ omits $\log(N/\delta)$ factors.}}}
    \renewcommand{\arraystretch}{1.4} 
     \resizebox{0.47\textwidth}{!}{
    \begin{tabular}{c|c|c} 
        \hline 

        \hline
        \textbf{Reference} & \textbf{Condition} & $\textsc{Err}_{N}(M)$ \\[0mm] \hline

        \rule{0pt}{5mm}
        \citet{ICML'24:Amit} & \textemdash & $\Ot\sbr{\frac{\sigma_\ell}{\sqrt{M}}}$ \\[1mm] \hline

        \rule{0pt}{5mm}
        \textbf{Ours [\pref{lemma:selection-variance}]} & \pref{eq:interpolation} & $\Ot\sbr{\frac{\sqrt{V_0}}{\sqrt{M}} + \frac{\sigma_\ell + V_1}{M}}$ \\[1mm] \hline

        \hline
    \end{tabular}}
    \label{table:ensemble}
\end{table}

In the following, we state our new model-ensemble lemma, which leverages a gap-dependent variance bound to achieve a sharper ensemble guarantee. The proof is in \pref{app:selection-variance}.
\begin{myLemma}
    \label{lemma:selection-variance}
    Given a zeroth-order oracle $\ellt(\cdot)$, as formalized in Assumption~\ref{assum:bounded-value-oracle}, we further assume that the objective $\ell(\cdot)$ admits a minimizer $\xs \in \argmin_{\x \in \X} \ell(\x)$, and $\forall \x \in \X$,
    \begin{equation}
        \label{eq:interpolation}
        \var[\ellt(\x) \given \x] \leq V_0 + V_1 (\ell(\x)-\ell(\xs)), 
    \end{equation}
    with some constants $ V_0, V_1 \ge 0$.
    Given $\x_1,\ldots,\x_N$, denote $\xb = \argmin_{i\in[N]}\sum_{j=1}^{M}\ellt_j(\x_i)$. 
    Then with probability at least $1-\delta$, it holds that
    \begin{align*}
    \ell(\xb) &- \ell(\xs) \leq
    3 \Big(\min_{i \in [N]} \ell(\x_i) - \ell(\xs)\Big)
    \\&+ \frac{\sqrt{32 V_0 \log(2N/\delta)}}{\sqrt{M}}
    + \frac{(8 \sigma_\ell + 24 V_1)\log(2 N/\delta)}{3M}
    .
    \end{align*}
\end{myLemma}
We provide a comparison between our new ensemble error and that from \citet{ICML'24:Amit} in~\pref{table:ensemble}.

For the assumption $\var[\ellt(\x) \given \x] \leq V_0 + V_1 (\ell(\x)-\ell(\xs))$ in~\pref{eq:interpolation}, canonical examples are over-parameterized neural networks and over-parameterized linear models~\citep{ICML'18:Ma,NeurIPS'23:Liu,evron2025from}.

\subsection{Discussions of Our Results}
\label{subsec:discussion}
In this part, we specify the ensemble error in Theorems~\ref{thm:con-search-gradient} and~\ref{thm:con-search-value}, and compare our universal rate,~\pref{thm:con-search-value}, with the best known~\citet[Theorem 2]{COLT'24:Kreisler}.

\textbf{Additional Ensemble Error.}~~
In Theorems~\ref{thm:con-search-gradient} and~\ref{thm:con-search-value}, the ensemble error $\textsc{Err}_N(T/N)$, in the worst case, is given by:
\begin{equation}
    \label{eq:ensemble-error-old}
    \textsc{Err}_N\sbr{\frac{T}{N}} = \sigma_\ell \sqrt{\frac{\log(N/\delta)}{T/N}}
\end{equation}
from~\citet[Lemma 9]{ICML'24:Amit}, restated in~\pref{lemma:ensemble}. 
Here, the number of candidates $N$ is a logarithmic factor.
Notably, although the ensemble error incorporates a zeroth-order variance $\sigma_\ell$, it does not depend on $d_0$, showing that it might be smaller than the variance term like $d_0 \Delta_\ell / \sqrt{T}$.

Moreover, using our new problem-dependent ensemble analysis in~\pref{lemma:selection-variance}, it holds that
\begin{equation}
    \label{eq:ensemble-error-new}
    \textsc{Err}_N\sbr{\frac{T}{N}} = \sqrt{\frac{V_0 \log(N/\delta)}{T/N}} + \frac{(\sigma_\ell + V_1)\log(N/\delta)}{T/N}.
\end{equation}
As $V_0 = \sigma_\ell^2$ and $V_1 = 0$ in the \emph{worst case}, \pref{eq:ensemble-error-new} strictly recovers the original result in~\pref{eq:ensemble-error-old}.
Moreover,~\pref{eq:ensemble-error-new} provides a gap-dependent perspective, based on the idea that the variance of the zeroth-order oracle may diminish when approaching the optimum. 
Intuitively, $V_0$ represents the variance at the optimum, and $V_1$ characterizes how the variance scales with the function value gap of $\ell(\cdot)-\ell(\xs)$.

For theoretical considerations, we compare the function values of the base algorithms' last iterates. 
While in practice, additional information, such as the decreasing behavior of the loss curve along the optimization trajectory, may help select the best candidate more effectively.

\textbf{Comparison with~\citet{COLT'24:Kreisler}.}~~
The \udog algorithm~\citep[Theorem 2]{COLT'24:Kreisler} is the best-known result for universal stochastic convex optimization.
Given the \emph{true} lower bound $r_\epsilon$ of $d_0$ and the gradient noise bound function $\hat\Delta(\cdot): \mathbb{R}^d \to \mathbb{R}_+$ (an upper bound of the true noise function $\Delta(\cdot)$ in~\pref{assum:bounded-gradient-oracle}), \udog ensures:
\begin{equation*}
    \label{eq:U-DOG}
    \O\sbr{C^{\textnormal{\udog}} \Big(\min\bbr{\frac{L_\ell d_0^2}{T^2}, \frac{Gd_0}{\sqrt{T}}} + \frac{d_0 \Delta_{2d_0}}{\sqrt{T}} + \frac{d_0\hat{\Delta}_{2d_0}}{T} \Big)}
\end{equation*}
with probability at least $1-\delta$, where the coefficient $\smash{C^{\textnormal{\udog}}}$:
\begin{equation*}
    \sbr{\log\frac{d_0}{r_\epsilon}}^2\sbr{\log\frac{1}{\delta}}^2\sbr{ \log T \frac{\hat{\Delta}_{2d_0} + \min\{L_\ell d_0^2, G d_0\}}{\ell(\x_0) - \ell(\xs)}}^4.
\end{equation*}
\udog and our method differ in various aspects, including the methodology, algorithmic conditions, and the final theoretical guarantees. 
In the following, we discuss these aspects in detail, trying to provide a clear comparison of the advantages and disadvantages between the two methods.

\textbf{Methodology:}~~We apply the grid search method, which is flexible and extendable since we use the base algorithm in a \emph{black-box} manner. 
Thus, we do not need to dive into the analytical details of \UniXGrad, and can replace it with any other base algorithm with the same convergence rate.
In contrast, \udog also applies \UniXGrad, but in a \emph{white-box} manner to fit the ``DoG'' analysis~\citep{ICML'23:DoG} to be free from $d_0$.

\textbf{Conditions:}~~\udog requires a lower bound for $d_0$ and the gradient noise bound function $\hat\Delta(\cdot)$. 
In contrast, our method needs a lower bound for the function value, which is more accessible in some cases.
It is worth mentioning that our algorithm does not require any knowledge related to $L_\ell$ and $d_0$, whereas the analysis of \udog holds only when given the true lower bound of $d_0$, and a true upper bound of $\Delta(\cdot)$.

\textbf{Convergence:}~~As summarized in~\pref{table:convex}, our result can adapt to the general \Holder smoothness, whereas \mbox{\udog} provides guarantees only for the separate cases of smoothness and non-smoothness. 
Besides, our coefficients of the common terms, such as the main convergence and the variance terms, are smaller, where a formal comparison is provided in~\pref{table:convex}.
However, our bound includes an ensemble error caused by the nature of the grid search framework, which does not appear in the rate of \udog.
\section{Experiments}
\label{sec:exp}

\def \xout {\x^{\text{out}}}
\def \xtuned {\x^{\text{tuned}}}

In this section, we conduct experiments, aiming to investigate two questions:
\emph{(i)} How does our parameter-free method perform compared to the tuned methods? \emph{(ii)} The algorithm design requires sufficient initial sampling to ensure theoretical tight convergence; will reducing the initial sampling appropriately have a significant impact on the results?
To this end, our experiments conduct two key messages:
\begin{itemize}[left=10pt, itemsep=0pt, topsep=0pt]
    \item[(1)] Our parameter-free methods obtain competitive convergence compared with optimally tuned methods.
    \item[(2)] Reducing the initial sample budget has little impact, and may improve the results thanks to the saved budget.
\end{itemize}

\paragraph{Setup.}
We use ResNet-18 model~\citep{He2016resnet} (\texttt{torchvision.models.resnet18}) with an added linear layer (for convex case, we use pre-trained ResNet-18 and only update the linear layer), CIFAR-10 dataset~\citep{krizhevsky2009cifar} (batch size $64$), and cross-entropy loss. We split the oracle budget $T=10^4$ into two parts: the initial sampling budget $M$ for determining the search range, and the training budget $T-M$. We then choose the candidate with the best performance (i.e., the minimum averaged gradient norm or loss within a sliding window of the last $100$ batches, following a common practical ``checkpoint'' approach) as the output of our method, denoted as $\xout$.
We vary algorithmic inputs (e.g., $L_\epsilon, F_\epsilon$ for non-convex case, and $d_\epsilon, L_\epsilon$ for convex case), and the initial sampling budget $M$ from $T/4$ to $T/4096$, to evaluate the sensitivity of our method.

\paragraph{Measure.}
To compare our method with the optimally tuned one, we define \emph{relative difference} $\rho$.
Our method outputs $\xout$. We perform a more fine-grained search to find the best-tuned hyper-parameters for base algorithm (i.e., SGD for~\pref{alg:non-convex}, and \UniXGrad for~\pref{alg:convex}), allocate the \emph{entire budget} $T$ to it, and get $\xtuned$. Then we define ${\rho \define \frac{\|\nabla \ell(\xout)\| - \| \nabla \ell(\xtuned)\|}{\|\nabla \ell(\xtuned)\|}}$ for non-convex case, and ${\rho \define \frac{\ell(\xout) - \ell(\xtuned)}{\ell(\xtuned)}}$ for convex case.
Then $\rho\ge 0$, and smaller $\rho$ indicates better performance.

\paragraph{Results of Non-convex Setting.}
For~\pref{alg:non-convex}, \textsc{Grasp-NC}, our results are shown in~\pref{table:nc}.
We can see that the values of $\rho$ lie between $0.11$ and $0.23$, meaning the actual impact of the log factor in our convergence rate guarantee only results in a multiplier of $1.11$ to $1.23$, justifying the claim that our method achieves competitive performance with the optimally tuned one. Moreover, as the initial sampling budget $M$ decreases, the performance does not noticeably degrade, which indicates that in practice, we can use a smaller $M$ to save oracle calls for training.

\begin{table*}[!t]
\centering
\caption{\textsc{Grasp-NC}. We set $\delta=0.05$, and vary inputs $(L_\epsilon,F_\epsilon)$ in $\{0.001, 0.01, 0.1\}\times \{0.001, 0.01, 0.1\}$.}
\label{table:nc}
\begin{tabular}{c|c|c|c|c|c|c|c}
    \hline
    \multicolumn{2}{c|}{Inputs} & \multicolumn{6}{c}{$\rho$} \\ \hline
    $L_\epsilon$ & $F_\epsilon$ & $M=T/4$ & $M=T/16$ & $M=T/64$ & $M=T/256$ & $M=T/1024$ & $M=T/4096$ \\ \hline
    0.001 & 0.001 & 0.1836 & 0.1699 & 0.1699 & 0.1699 & 0.1699 & 0.1699 \\ \hline
    0.001 & 0.01 & 0.1468 & 0.1468 & 0.1468 & 0.1468 & 0.1468 & 0.1468 \\ \hline
    0.001 & 0.1 & 0.1468 & 0.1468 & 0.1468 & 0.1468 & 0.1468 & 0.1468 \\ \hline
    0.01 & 0.001 & 0.2336 & 0.2336 & 0.2336 & 0.2336 & 0.2336 & 0.2336 \\ \hline
    0.01 & 0.01 & 0.2094 & 0.1772 & 0.1681 & 0.1772 & 0.2094 & 0.2094 \\ \hline
    0.01 & 0.1 & 0.2094 & 0.1572 & 0.1557 & 0.1557 & 0.1583 & 0.1772 \\ \hline
    0.1 & 0.001 & 0.2264 & 0.1365 & 0.1365 & 0.1365 & 0.1365 & 0.1657 \\ \hline
    0.1 & 0.01 & 0.2023 & 0.1141 & 0.1139 & 0.1139 & 0.1141 & 0.1141 \\ \hline
    0.1 & 0.1 & 0.1141 & 0.1139 & 0.1139 & 0.1139 & 0.1139 & 0.1141 \\ \hline
\end{tabular}
\end{table*}

\paragraph{Results of Convex Setting.}
For~\pref{alg:convex}, \textsc{Grasp-C}, our results are shown in~\pref{table:convex-i} and~\pref{table:convex-ii}.
Both \textsc{Option-I} and \textsc{Option-II} show similar performance, as the different options only influence the upper bound of the search range, and their best candidates consistently use the same hyper-parameters.
In both tables, small values of $\rho$ (between $0.08$ and $0.16$) also indicate that our method achieves competitive performance with the optimally tuned one. Moreover, the performance is not sensitive to the initial sampling budget $M$, \mbox{again suggesting that we can use a smaller $M$ in practice.}

\begin{table*}[!t]
\centering
\caption{\textsc{Grasp-C Option-I}. We vary inputs $(d_\epsilon, L_\epsilon)$ in $\{0.001, 0.01, 0.1\}\times \{0.001, 0.01, 0.1\}$.}
\label{table:convex-i}
\begin{tabular}{c|c|c|c|c|c|c|c}
    \hline
    \multicolumn{2}{c|}{Inputs} & \multicolumn{6}{c}{$\rho$} \\ \hline
    $d_\epsilon$ & $L_\epsilon$ & $M=T/4$ & $M=T/16$ & $M=T/64$ & $M=T/256$ & $M=T/1024$ & $M=T/4096$ \\ \hline
    0.001 & 0.001 & 0.1640 & 0.1419 & 0.1419 & 0.1419 & 0.1419 & 0.1419 \\ \hline
    0.001 & 0.01 & 0.1631 & 0.1419 & 0.1419 & 0.1419 & 0.1419 & 0.1419 \\ \hline
    0.001 & 0.1 & 0.1610 & 0.1419 & 0.1419 & 0.1408 & 0.1408 & 0.1411 \\ \hline
    0.01 & 0.001 & 0.1303 & 0.1210 & 0.1190 & 0.1190 & 0.1190 & 0.1190 \\ \hline
    0.01 & 0.01 & 0.1303 & 0.1197 & 0.1190 & 0.1190 & 0.1190 & 0.1190 \\ \hline
    0.01 & 0.1 & 0.1303 & 0.1197 & 0.1190 & 0.1190 & 0.1187 & 0.1187 \\ \hline
    0.1 & 0.001 & 0.0920 & 0.0873 & 0.0873 & 0.0873 & 0.0873 & 0.0873 \\ \hline
    0.1 & 0.01 & 0.0920 & 0.0873 & 0.0873 & 0.0866 & 0.0866 & 0.0866 \\ \hline
    0.1 & 0.1 & 0.0898 & 0.0873 & 0.0866 & 0.0866 & 0.0866 & 0.0866 \\ \hline
\end{tabular}
\end{table*}

\begin{table*}[!t]
\centering
\caption{\textsc{Grasp-C Option-II}. We set $\ell_\epsilon^\star=0$, and vary inputs $(d_\epsilon, L_\epsilon)$ in $\{0.001, 0.01, 0.1\}\times \{0.001, 0.01, 0.1\}$.}
\label{table:convex-ii}
\begin{tabular}{c|c|c|c|c|c|c|c}
    \hline
    \multicolumn{2}{c|}{Inputs} & \multicolumn{6}{c}{$\rho$} \\ \hline
    $d_\epsilon$ & $L_\epsilon$ & $M=T/4$ & $M=T/16$ & $M=T/64$ & $M=T/256$ & $M=T/1024$ & $M=T/4096$ \\ \hline
    0.001 & 0.001 & 0.1640 & 0.1419 & 0.1419 & 0.1419 & 0.1419 & 0.1419 \\ \hline
    0.001 & 0.01 & 0.1631 & 0.1419 & 0.1419 & 0.1419 & 0.1419 & 0.1419 \\ \hline
    0.001 & 0.1 & 0.1616 & 0.1419 & 0.1419 & 0.1411 & 0.1411 & 0.1411 \\ \hline
    0.01 & 0.001 & 0.1303 & 0.1210 & 0.1190 & 0.1190 & 0.1190 & 0.1190 \\ \hline
    0.01 & 0.01 & 0.1303 & 0.1200 & 0.1190 & 0.1190 & 0.1190 & 0.1190 \\ \hline
    0.01 & 0.1 & 0.1303 & 0.1197 & 0.1190 & 0.1190 & 0.1187 & 0.1187 \\ \hline
    0.1 & 0.001 & 0.0920 & 0.0873 & 0.0873 & 0.0873 & 0.0873 & 0.0873 \\ \hline
    0.1 & 0.01 & 0.0920 & 0.0873 & 0.0873 & 0.0866 & 0.0866 & 0.0866 \\ \hline
    0.1 & 0.1 & 0.0898 & 0.0873 & 0.0866 & 0.0866 & 0.0866 & 0.0866 \\ \hline
\end{tabular}
\end{table*}

\section{Conclusion}
\label{sec:conclusion}
In this paper, we focus on achieving full parameter-freeness in stochastic optimization.
To this end, we propose a novel self-bounding analysis technique that effectively derives searching ranges for the unknown problem parameters.
We validate the effectiveness of our framework in both non-convex and convex settings.
In the non-convex setting, we design a fully parameter-free method with the optimal convergence rate up to logarithmic factors.
In the convex setting, our method is also competitive with the state-of-the-art results regarding parameter-freeness and convergence.
We also provide a new guarantee for model ensembles based on a gap-dependent variance characterization, which may be of independent interest.

We highlight two interesting directions for future work.
The first is to apply our framework to other settings, as our self-bounding analysis is quite general and could potentially be extended to parameter-free methods tailored for different scenarios.
The second is to develop more effective and efficient ensemble methods to reduce ensemble error and minimize oracle calls, thereby enhancing performance.


\section*{Acknowledgements} 
PZ is grateful to Ashok Cutkosky for initial discussions on grid search.

\bibliography{reference}

\begin{thebibliography}{29}
\providecommand{\natexlab}[1]{#1}
\providecommand{\url}[1]{\texttt{#1}}
\expandafter\ifx\csname urlstyle\endcsname\relax
  \providecommand{\doi}[1]{doi: #1}\else
  \providecommand{\doi}{doi: \begingroup \urlstyle{rm}\Url}\fi

\bibitem[Attia \& Koren(2023)Attia and Koren]{attia2023sgd}
Attia, A. and Koren, T.
\newblock Sgd with adagrad stepsizes: Full adaptivity with high probability to unknown parameters, unbounded gradients and affine variance.
\newblock In \emph{Proceedings of the 40th International Conference on Machine Learning (ICML)}, pp.\  1147--1171. PMLR, 2023.

\bibitem[Attia \& Koren(2024)Attia and Koren]{ICML'24:Amit}
Attia, A. and Koren, T.
\newblock How free is parameter-free stochastic optimization?
\newblock In \emph{Proceedings of the 41st International Conference on Machine Learning (ICML)}, pp.\  2009--2034, 2024.

\bibitem[Bottou et~al.(2018)Bottou, Curtis, and Nocedal]{bottou2018optimization}
Bottou, L., Curtis, F.~E., and Nocedal, J.
\newblock Optimization methods for large-scale machine learning.
\newblock \emph{SIAM Review}, 60\penalty0 (2):\penalty0 223--311, 2018.

\bibitem[Cutkosky(2019)]{ICML'19:Cutkosky}
Cutkosky, A.
\newblock Anytime online-to-batch, optimism and acceleration.
\newblock In \emph{Proceedings of the 36th International Conference on Machine Learning (ICML)}, pp.\  1446--1454, 2019.

\bibitem[Defazio \& Mishchenko(2023)Defazio and Mishchenko]{ICML'23:Defazio}
Defazio, A. and Mishchenko, K.
\newblock Learning-rate-free learning by {D}-adaptation.
\newblock In \emph{Proceedings of the 40th International Conference on Machine Learning (ICML)}, pp.\  7449--7479, 2023.

\bibitem[Devolder et~al.(2014)Devolder, Glineur, and Nesterov]{devolder2014first}
Devolder, O., Glineur, F., and Nesterov, Y.
\newblock First-order methods of smooth convex optimization with inexact oracle.
\newblock \emph{Mathematical Programming}, 146:\penalty0 37--75, 2014.

\bibitem[Evron et~al.(2025)Evron, Levinstein, Schliserman, Sherman, Koren, Soudry, and Srebro]{evron2025from}
Evron, I., Levinstein, R., Schliserman, M., Sherman, U., Koren, T., Soudry, D., and Srebro, N.
\newblock From continual learning to {SGD} and back: Better rates for continual linear models.
\newblock \emph{arXiv preprint}, arXiv:2504.04579, 2025.

\bibitem[Faw et~al.(2022)Faw, Tziotis, Caramanis, Mokhtari, Shakkottai, and Ward]{faw2022power}
Faw, M., Tziotis, I., Caramanis, C., Mokhtari, A., Shakkottai, S., and Ward, R.
\newblock The power of adaptivity in {SGD}: Self-tuning step sizes with unbounded gradients and affine variance.
\newblock In \emph{Proceedings of the 35th Annual Conference on Learning Theory (COLT)}, pp.\  313--355, 2022.

\bibitem[Ghadimi \& Lan(2013)Ghadimi and Lan]{ghadimi2013stochastic}
Ghadimi, S. and Lan, G.
\newblock Stochastic first-and zeroth-order methods for nonconvex stochastic programming.
\newblock \emph{SIAM Journal on Optimization}, 23\penalty0 (4):\penalty0 2341--2368, 2013.

\bibitem[He et~al.(2016)He, Zhang, Ren, and Sun]{He2016resnet}
He, K., Zhang, X., Ren, S., and Sun, J.
\newblock Deep residual learning for image recognition.
\newblock In \emph{Proceedings of the IEEE conference on computer vision and pattern recognition}, pp.\  770--778, 2016.

\bibitem[Ivgi et~al.(2023)Ivgi, Hinder, and Carmon]{ICML'23:DoG}
Ivgi, M., Hinder, O., and Carmon, Y.
\newblock {DoG} is {SGD}'s best friend: {A} parameter-free dynamic step size schedule.
\newblock In \emph{Proceedings of the 40th International Conference on Machine Learning (ICML)}, pp.\  14465--14499, 2023.

\bibitem[Kavis et~al.(2019)Kavis, Levy, Bach, and Cevher]{NeurIPS'19:UniXGrad}
Kavis, A., Levy, K.~Y., Bach, F.~R., and Cevher, V.
\newblock {UniXGrad}: {A} universal, adaptive algorithm with optimal guarantees for constrained optimization.
\newblock In \emph{Advances in Neural Information Processing Systems 32 (NeurIPS)}, pp.\  6257--6266, 2019.

\bibitem[Kavis et~al.(2022)Kavis, Levy, and Cevher]{kavis2022high}
Kavis, A., Levy, K.~Y., and Cevher, V.
\newblock High probability bounds for a class of nonconvex algorithms with adagrad stepsize.
\newblock In \emph{Proceedings of the 10th International Conference on Learning Representations (ICLR)}, 2022.

\bibitem[Khaled \& Jin(2024)Khaled and Jin]{khaled2024tuningfree}
Khaled, A. and Jin, C.
\newblock Tuning-free stochastic optimization.
\newblock In \emph{Proceedings of the 41st International Conference on Machine Learning (ICML)}, pp.\  23622--23661, 2024.

\bibitem[Kingma \& Ba(2015)Kingma and Ba]{ICLR'14:Adam}
Kingma, D.~P. and Ba, J.
\newblock {Adam}: {A} method for stochastic optimization.
\newblock In \emph{Proceedings of the 3rd International Conference on Learning Representations (ICLR)}, 2015.

\bibitem[Kreisler et~al.(2024)Kreisler, Ivgi, Hinder, and Carmon]{COLT'24:Kreisler}
Kreisler, I., Ivgi, M., Hinder, O., and Carmon, Y.
\newblock Accelerated parameter-free stochastic optimization.
\newblock In \emph{Proceedings of the 37th Annual Conference on Learning Theory (COLT)}, pp.\  3257--3324, 2024.

\bibitem[Krizhevsky et~al.(2009)Krizhevsky, Hinton, et~al.]{krizhevsky2009cifar}
Krizhevsky, A., Hinton, G., et~al.
\newblock Learning multiple layers of features from tiny images.
\newblock 2009.

\bibitem[Lan(2020)]{Book:Lan-SCO}
Lan, G.
\newblock \emph{First-order and Stochastic Optimization Methods for Machine Learning}, volume~1.
\newblock Springer, 2020.

\bibitem[Li \& Lan(2025)Li and Lan]{li2025simple}
Li, T. and Lan, G.
\newblock A simple uniformly optimal method without line search for convex optimization.
\newblock \emph{Mathematical Programming}, pp.\  1--38, 2025.

\bibitem[Liu et~al.(2023{\natexlab{a}})Liu, Drusvyatskiy, Belkin, Davis, and Ma]{NeurIPS'23:Liu}
Liu, C., Drusvyatskiy, D., Belkin, M., Davis, D., and Ma, Y.
\newblock Aiming towards the minimizers: Fast convergence of {SGD} for overparametrized problems.
\newblock In \emph{Advances in Neural Information Processing Systems 36 (NeurIPS)}, pp.\  60748--60767, 2023{\natexlab{a}}.

\bibitem[Liu et~al.(2023{\natexlab{b}})Liu, Nguyen, Nguyen, Ene, and Nguyen]{liu2023high}
Liu, Z., Nguyen, T.~D., Nguyen, T.~H., Ene, A., and Nguyen, H.
\newblock High probability convergence of stochastic gradient methods.
\newblock In \emph{Proceedings of the 40th International Conference on Machine Learning (ICML)}, pp.\  21884--21914. PMLR, 2023{\natexlab{b}}.

\bibitem[Loshchilov \& Hutter(2019)Loshchilov and Hutter]{ICLR'19:AdamW}
Loshchilov, I. and Hutter, F.
\newblock Decoupled weight decay regularization.
\newblock In \emph{Proceedings of the 7th International Conference on Learning Representations (ICLR)}, 2019.

\bibitem[Ma et~al.(2018)Ma, Bassily, and Belkin]{ICML'18:Ma}
Ma, S., Bassily, R., and Belkin, M.
\newblock The power of interpolation: Understanding the effectiveness of {SGD} in modern over-parametrized learning.
\newblock In \emph{Proceedings of the 35th International Conference on Machine Learning (ICML)}, pp.\  3325--3334, 2018.

\bibitem[Nesterov(2015)]{nesterov2015universal}
Nesterov, Y.
\newblock Universal gradient methods for convex optimization problems.
\newblock \emph{Mathematical Programming}, 152\penalty0 (1):\penalty0 381--404, 2015.

\bibitem[Robbins \& Monro(1951)Robbins and Monro]{SGD}
Robbins, H. and Monro, S.
\newblock A stochastic approximation method.
\newblock \emph{The Annals of Mathematical Statistics}, 22\penalty0 (3):\penalty0 400 -- 407, 1951.

\bibitem[Schulman et~al.(2017)Schulman, Wolski, Dhariwal, Radford, and Klimov]{schulman2017DPO}
Schulman, J., Wolski, F., Dhariwal, P., Radford, A., and Klimov, O.
\newblock Proximal policy optimization algorithms.
\newblock \emph{arXiv preprint arXiv:1707.06347}, 2017.

\bibitem[Shao et~al.(2024)Shao, Wang, Zhu, Xu, Song, Bi, Zhang, Zhang, Li, Wu, et~al.]{shao2024GRPO}
Shao, Z., Wang, P., Zhu, Q., Xu, R., Song, J., Bi, X., Zhang, H., Zhang, M., Li, Y., Wu, Y., et~al.
\newblock {DeepSeekMath}: Pushing the limits of mathematical reasoning in open language models.
\newblock \emph{arXiv preprint arXiv:2402.03300}, 2024.

\bibitem[Wainwright(2019)]{wainwright2019high}
Wainwright, M.~J.
\newblock \emph{High-dimensional Statistics: A Non-asymptotic Viewpoint}, volume~48.
\newblock Cambridge university press, 2019.

\bibitem[Zhao et~al.(2025)Zhao, Yan, Levy, and Zhao]{NeurIPS'25:GV4OPT}
Zhao, Y., Yan, Y.-H., Levy, K.~Y., and Zhao, P.
\newblock Gradient-variation online adaptivity for accelerated optimization with hölder smoothness.
\newblock In \emph{Advances in Neural Information Processing Systems 38 (NeurIPS)}, pp.\  to appear, 2025.

\end{thebibliography}
\bibliographystyle{icml2026}

\newpage
\onecolumn
\appendix

\section{Omitted Details for Section~\ref{sec:non-convex}}
\label{appendix:proof-non-convex}

\begin{proof}[Proof of~\pref{thm:non-convex}]
By~\pref{lemma:SGD}, given constant step size $\eta\le 1/(2L_\ell)$, SGD ensures that, with probability at least $1-\delta$:
\begin{equation*}
    \frac{1}{T}\sum_{t=1}^{T}\norm{\nabla\ell(\x_t)}^2 \le \frac{4F_\ell}{\eta T} + 4\eta L_\ell \Delta_{\ell}^2 + \frac{12\Delta_{\ell}^2\log\frac{1}{\delta}}{T},
\end{equation*}
where $F_\ell\triangleq \ell(\x_1) - \min_{\x\in\R^d}\ell(\x)$.
Now with any $L_\epsilon, \Delta_\epsilon, F_\epsilon >0$, define $\bar{L}_\ell \define \max\{L_\ell, L_\epsilon\}, \bar{\Delta}_{\ell} \define \max\{\Delta_\ell, \Delta_\epsilon\}$ and $\bar{F}_\ell \define \max\{F_\ell, F_\epsilon\}$. 
Then SGD with step size $\eta\le 1/(2\bar{L}_\ell) \le 1/(2L_\ell)$ ensures that, with probability at least $1-\delta$:
\begin{equation}
    \label{eq:non-convex-target}
    \frac{1}{T}\sum_{t=1}^{T}\norm{\nabla\ell(\x_t)}^2 \le \frac{4\bar{F}_\ell}{\eta T} + 4\eta \bar{L}_\ell \bar{\Delta}_{\ell}^2 + \frac{12\Delta_{\ell}^2\log\frac{1}{\delta}}{T} = \underbrace{\O\sbr{ \sqrt{\frac{\bar{L}_\ell\bar{F}_\ell \bar{\Delta}_\ell^2}{T}} + \frac{\bar{L}_\ell\bar{F}_\ell}{T} + \frac{\Delta_{\ell}^2\log\frac{1}{\delta}}{T} }}_{\textsc{Target}},
\end{equation}
where the optimal tuning for the step size $\eta$ is $\smash{\min\{1/(2\bar{L}_\ell),\sqrt{\bar{F}_\ell/(\bar{L}_\ell\bar{\Delta}_\ell^2 T)}\}}$.

Let $\x^0$ be one of the candidates. 
By sampling $\g(\x^0)$ for $\smash{M_0\define\lfloor\frac{T}{4}\rfloor}$ times and define the average $\smash{\gh^0\define\frac{1}{M_0}\sum_{t=1}^{M_0}\g_t(\x^0)}$, then applying~\pref{lemma:sample-gradient}, with probability at least $1-\delta$:
\begin{equation}
    \label{eq:non-convex-self-bounding}
    \norm{\nabla \ell(\x^0)}^2 \le 2\norm{\gh^0}^2 + 2\norm{\gh^0 - \nabla \ell(\x^0)}^2 \le 2\norm{\gh^0}^2 + \O\sbr{ \frac{(\Delta(\x^0))^2\log\frac{1}{\delta}}{T} }.
\end{equation}

\paragraph{Self-bounding Analysis.}
Using~\pref{eq:non-convex-self-bounding} as the comparator, next we consider the condition when target convergence $\textsc{Target}$ in~\pref{eq:non-convex-target} becomes vacuous. 
For three parameters, ${L}_\ell,{\Delta}_\ell$ and ${F}_\ell$:
\begin{itemize}
    \item Define $L_{\max}\define \frac{\norm{\gh^0}^2 T}{F_\epsilon}$, then ${L}_\ell>L_{\max}$ implies that $\norm{\gh^0}^2 < \frac{{L}_\ell F_\epsilon}{T} \le \textsc{Target}$, which makes~\pref{eq:non-convex-target} vacuous.
    \item Define $F_{\max}\define \frac{\norm{\gh^0}^2 T}{L_\epsilon}$, then ${F}_\ell>F_{\max}$ implies that $\norm{\gh^0}^2 < \frac{L_\epsilon {F}_\ell}{T} \le \textsc{Target}$, which makes~\pref{eq:non-convex-target} vacuous.
    \item Define $\Delta_{\max}^2\define\frac{\norm{\gh^0}^2 T}{\log\frac{1}{\delta}}$, then ${\Delta}_\ell^2>\Delta_{\max}^2$ implies that $\norm{\gh^0}^2 < \frac{{\Delta}_\ell^2\log\frac{1}{\delta}}{T} \le \textsc{Target}$, which makes~\pref{eq:non-convex-target} vacuous.
\end{itemize}
In conclude, in case of ${L}_\ell>L_{\max}$ or ${F}_\ell>F_{\max}$ or ${\Delta}_\ell^2>\Delta_{\max}^2$, we have, with probability at least $1-\delta$:
\begin{equation}
    \label{eq:non-convex-trivial}
    \norm{\nabla \ell(\x^0)}^2 \overset{\eqref{eq:non-convex-self-bounding}}{\le} 2\norm{\gh^0}^2 + \O\sbr{ \frac{(\Delta(\x^0))^2\log\frac{1}{\delta}}{T} } \le \O\sbr{ \sqrt{\frac{\bar{L}_\ell\bar{F}_\ell \bar{\Delta}_\ell^2}{T}} + \frac{\bar{L}_\ell\bar{F}_\ell}{T} + \frac{\Delta_{\ell}^2\log\frac{1}{\delta}}{T} +  \frac{(\Delta(\x^0))^2\log\frac{1}{\delta}}{T} }.
\end{equation}
In the following, we consider the case of ${L}_\ell \le L_{\max}$ and ${F}_\ell\le F_{\max}$ and ${\Delta}_\ell^2\le \Delta_{\max}^2$. That is, we have the conditions:
\begin{equation*}
    \bar{L}_\ell \in \mbr{L_\epsilon, \max\{L_\epsilon,L_{\max}\}}, \quad \bar{F}_\ell\in\mbr{F_\epsilon,\max\{F_\epsilon, F_{\max}\}}, \quad \bar{\Delta}_\ell^2\in\mbr{ \Delta_\epsilon^2, \max\{\Delta_\epsilon^2,\Delta_{\max}^2\} },
\end{equation*}
which gives us the grid search range for the optimal step size tuning in target convergence~\pref{eq:non-convex-target}. 
Specifically, since now for each SGD algorithm we apply a budget $T'\equiv\lfloor \frac{T}{2N} \rfloor$, the optimal tuning $\eta_\star$ is:
\begin{equation*}
    \eta_\star = \min\bbr{ \frac{1}{2\bar{L}_\ell}, \sqrt{ \frac{2\bar{F}_\ell}{\bar{L}_\ell \bar{\Delta}_\ell^2 T'} } } = \min\bbr{ \frac{1}{2\bar{L}_\ell}, \sqrt{ \frac{2\bar{F}_\ell N}{\bar{L}_\ell \bar{\Delta}_\ell^2 T} } },
\end{equation*}
which has the following lower bound $\eta_{\min}$ and upper bound $\eta_{\max}$:
\begin{equation}
    \label{eq:non-convex-eta-range}
    \eta_{\min} = \min\bbr{ \frac{1}{2\max\{L_\epsilon, L_{\max}\}}, \sqrt{ \frac{2 F_\epsilon}{\max\{L_\epsilon, L_{\max}\}\cdot \max\{\Delta_{\epsilon}^2, \Delta_{\max}^2\} \cdot T} } } \le \eta_\star \le \eta_{\max} = \frac{1}{2L_\epsilon}.
\end{equation}
We set $\eta_i = \eta_{\min}2^i$, hence $N=\lceil \log_2 \frac{\eta_{\max}}{\eta_{\min}} \rceil$, and there exists $\is = \lceil \log_2 \frac{\eta_\star}{\eta_{\min}} \rceil \in[N]$ such that $\eta_{\is}/2\le \eta_\star \le \eta_{\is}$, then by~\pref{eq:non-convex-target}, with probability at least $1-\delta$:
\begin{equation}
    \frac{1}{T'}\sum_{t=1}^{T'}\norm{\nabla\ell(\x_t^{\is})}^2 \le \O\sbr{ \sqrt{\frac{\bar{L}_\ell\bar{F}_\ell \bar{\Delta}_\ell^2}{T'} } + \frac{\bar{L}_\ell\bar{F}_\ell  + \Delta_{\ell}^2\log\frac{1}{\delta}}{T'} } = \O\sbr{ \sqrt{\frac{\bar{L}_\ell\bar{F}_\ell \bar{\Delta}_\ell^2}{T/N} } + \frac{\bar{L}_\ell\bar{F}_\ell + \Delta_{\ell}^2\log\frac{1}{\delta}}{T/N} } \label{eq:non-convex-non-trivial}.
\end{equation}

\paragraph{Ensemble Error.}
By~\pref{lemma:ensemble-effective}, our ensemble method ensures that, with probability at least $1-2\delta$:
\begin{equation*}
    \norm{\nabla\ell(\xb)}^2 \le 8\min \bbr{\norm{\nabla\ell(\x^0)}^2, \min_{i\in[N]}\frac{1}{T_i}\sum_{t=1}^{T_i}\norm{\nabla\ell(\x_t^i)}^2} + \O\sbr{\frac{N(\log\frac{N}{\delta})(\log\frac{1}{\delta})\Delta_\ell^2}{T}},
\end{equation*}
therefore, combining~\pref{eq:non-convex-trivial} and~\pref{eq:non-convex-non-trivial}, we conclude that, with probability at least $1-3\delta$:
\begin{align*}
    \norm{\nabla\ell(\xb)}^2 &\le \O\sbr{ \sqrt{\frac{\bar{L}_\ell\bar{F}_\ell \bar{\Delta}_\ell^2}{T/N} } + \frac{\bar{L}_\ell\bar{F}_\ell + \Delta_{\ell}^2\log\frac{1}{\delta}}{T/N} + \frac{(\Delta(\x^0))^2\log\frac{1}{\delta}}{T} + \frac{N(\log\frac{N}{\delta})(\log\frac{1}{\delta})\Delta_\ell^2}{T} } \\
    &= \O\sbr{ \sqrt{\frac{\bar{L}_\ell\bar{F}_\ell \bar{\Delta}_\ell^2}{T}\log\frac{\eta_{\max}}{\eta_{\min}} } + \frac{\bar{L}_\ell\bar{F}_\ell + \Delta_{\ell}^2(\log\frac{1}{\delta})^2}{T}\log\frac{\eta_{\max}}{\eta_{\min}}},
\end{align*}
where we treat double-logarithmic factor $\log(\log\frac{\eta_{\max}}{\eta_{\min}})$ as constant. 
Finally, by~\pref{eq:non-convex-eta-range}, we have:
\begin{align*}
    \frac{\eta_{\max}}{\eta_{\min}} &= \O\sbr{ \frac{1}{L_\epsilon} \cdot \max\bbr{ \max\{L_\epsilon, L_{\max}\}, \sqrt{\frac{1}{F_\epsilon}\max\{L_\epsilon, L_{\max}\}\cdot \max\{\Delta_\epsilon^2, \Delta_{\max}^2\} \cdot T} } } \\
    &= \O\sbr{ \max\bbr{ \max\bbr{1, \frac{\norm{\gh^0}^2 T}{L_\epsilon F_\epsilon}}, \sqrt{\frac{1}{F_\epsilon}\max\bbr{1, \frac{\norm{\gh^0}^2 T}{L_\epsilon F_\epsilon}}\cdot \max\bbr{\Delta_\epsilon^2, \frac{\norm{\gh^0}^2 T}{\log\frac{1}{\delta}}} \cdot T} } } \\
    &= \O\sbr{ \poly \sbr{ T, \norm{\gh^0}, \frac{1}{L_\epsilon}, \frac{1}{F_\epsilon}, \Delta_\epsilon, \frac{1}{\log\frac{1}{\delta}} } },
\end{align*}
by which we can set $\Delta_\epsilon = 0_+$ (appears only in analysis), treat double-logarithmic factor as constant, then finish proof with:
\begin{equation*}
    \log \frac{\eta_{\max}}{\eta_{\min}} = \O\sbr{\log_+ \frac{T \norm{\gh^0}}{L_\epsilon F_\epsilon} }.
\end{equation*}
\end{proof}

\begin{myLemma}[SGD for Non-Convex Smooth Optimization~(Lemma 1 of \citet{ICML'24:Amit})]
    \label{lemma:SGD}
    Under~\pref{assum:bounded-gradient-oracle}, assume that $\ell(\x)$ is $L_\ell$-smooth and lower bounded by $\ell^\star$. 
    The SGD algorithm that updates by $\x_{t+1} = \x_t - \eta \g(\x_t)$ for all $t\in[T]$, when $\eta\le 1/(2L_\ell)$, ensures that, for any $\delta\in(0,1)$, with probability at least $1-\delta$:
    \begin{equation*}
        \frac{1}{T}\sum_{t=1}^{T}\norm{\nabla\ell(\x_t)}^2 \le \frac{4(\ell(\x_1)-\ell^\star)}{\eta T} + 4\eta L_\ell \Delta_{\ell}^2 + \frac{12\Delta_{\ell}^2\log\frac{1}{\delta}}{T}.
    \end{equation*}
\end{myLemma}

\begin{myLemma}[Lemma 2 of \citet{ICML'24:Amit}]
    \label{lemma:ensemble-gradient}
    Under~\pref{assum:bounded-gradient-oracle},
    given $N$ candidates $\x_1,\cdots, \x_N$, let $\xb = \argmin_{i\in[N]}\norm{\sum_{j=1}^{M}\g_j(\x_i)}$, where $\g_j(\x_i)$ means the $j$-th independent gradient query at $\x_i$. 
    Then for any $\delta\in(0,1)$, with probability at least $1-\delta$:
    \begin{equation*}
        \norm{\nabla \ell(\xb)}^2 \le 4\min_{i\in[N]}\norm{\nabla \ell(\x_i)}^2 + \frac{24(1+3\log\frac{N}{\delta})\Delta_\ell^2}{M}.
    \end{equation*}
\end{myLemma}

The following lemma is an abstraction of the ensemble method in~\citet[Theorem 1]{ICML'24:Amit}.
\begin{myLemma}
    \label{lemma:ensemble-effective}
    Under~\pref{assum:bounded-gradient-oracle}, assume that there are total $T$ gradient Oracle budget and $N$ instances of algorithms, and the $i$-th algorithm has a trajectory $\x_1^i,\dots,\x_{T_i}^i$. 
    Given $\delta\in(0,1/2)$, for each $i\in[N]$, uniformly at random select $K=\lceil\log_2\frac{1}{\delta}\rceil$ indices from $[T_i]$, denoted by $\K_i$, and let candidates set $\S = \{\x_k^i:i\in[N],k\in \K_i\}$. 
    Applying ensemble method in~\pref{lemma:ensemble-gradient} to obtain a selection $\xb$ from $\S$ with $\lfloor M=T/(KN)\rfloor$, ensures that, with probability at least $1-2\delta$,
    \begin{equation*}
        \norm{\nabla \ell(\xb)}^2 \le 8\min_{i\in[N]}\frac{1}{T_i}\sum_{t=1}^{T_i}\norm{\nabla\ell(\x_t^i)}^2 + \O\sbr{\frac{N(\log\frac{N}{\delta})(\log\frac{1}{\delta})\Delta_\ell^2}{T}}.
    \end{equation*}
\end{myLemma}

\begin{proof}[Proof of~\pref{lemma:ensemble-effective}]
By Markov's inequality, with probability at least $1/2$, a uniformly at random index $k\in[T_i]$ satisfies:
\begin{equation}
    \label{eq:ensemble-effective-sample}
    \norm{\nabla \ell(\x_k^i)}^2 \le 2\E_t\mbr{ \norm{\nabla \ell(\x_t^i)}^2 } = \frac{2}{T_i}\sum_{t=1}^{T_i}\norm{\nabla\ell(\x_t^i)}^2.
\end{equation}
Then as we sample $K\triangleq\lceil\log_2\frac{1}{\delta}\rceil$ random indices, with probability at least $1-\delta$, we add at least one point with the guarantee~\pref{eq:ensemble-effective-sample} to the candidate set $\S$. 
Therefore, there are total $KN$ candidates in $\S$. 
Applying~\pref{lemma:ensemble-gradient} with $M=\lfloor T/(KN)\rfloor$, then, with probability at least $1-2\delta$, the selected candidate $\xb$ satisfies:
\begin{equation*}
    \norm{\nabla \ell(\xb)}^2 \le 4\min_{\x\in\S}\norm{\nabla \ell(\x)}^2 + \frac{24KN(1+3\log\frac{KN}{\delta})\Delta_\ell^2}{T} 
    \overset{\eqref{eq:ensemble-effective-sample}}{\le} 8\min_{i\in[N]}\frac{1}{T_i}\sum_{t=1}^{T_i}\norm{\nabla\ell(\x_t^i)}^2 + \O\sbr{\frac{N(\log\frac{N}{\delta})(\log\frac{1}{\delta})\Delta_\ell^2}{T}},
\end{equation*}
which completes the proof.
\end{proof}

\section{Omitted Details for Section~\ref{sec:convex}}
In this section, we provide the omitted details for \pref{sec:convex}, including the proofs of Theorems~\ref{thm:con-search-gradient}, \ref{thm:con-search-value}, \pref{lemma:selection-variance}, a universal convergence rate with less prior knowledge but a worse $\O(\sigma_\ell / T^{1/4})$ ensemble error in \pref{app:less-prior}.
~Finally, in \pref{app:cvx-base}, we provide a self-contained analysis of \UniXGrad.

\subsection{Proof of~\pref{thm:con-search-gradient}}
\label{appendix:proof-con-search-gradient}

\begin{proof}
We independently sample $\g(\x^0)$ for $\frac{T}{8}$ times, define $\gh^0 = \frac{8}{T}\sum_{t=1}^{T/8}\g_t(\x^0)$ where $\g_t(\x^0)$ means the $t$-th gradient query on $\x^0$. 
By~\pref{lemma:sample-gradient}, with probability at least $1-\delta$:
\begin{equation*}
    \norm{\gh^0 - \nabla \ell(\x^0)} \le \O\sbr{ \frac{\Delta(\x^0)\sqrt{\log_+\frac{1}{\delta}}}{\sqrt{T}} }.
\end{equation*}
Therefore, with probability at least $1-\delta$:
\begin{equation}
    \label{eq:proof-con-search-gradient-x0}
    \ell(\x^0) - \ell(\xs) \le \norm{\nabla \ell(\x^0)} d_0 \le \sbr{\norm{\gh^0 } + \norm{\gh^0 - \nabla \ell(\x^0)}} d_0 
    \le \norm{\gh^0 } d_0 + \O\sbr{ \frac{\Delta(\x^0)d_0\sqrt{\log_+\frac{1}{\delta}}}{\sqrt{T}}}. 
\end{equation}
Let $\textsc{Err}_{N}(M)$ be ensemble error given $N$ candidates with $M$ oracle budget for each, with probability at least $1-\delta$, as stated in~\pref{subsec:selection}. That means with probability at least $1-\delta$, for any $0\le \is \le N$,
\begin{equation}
    \label{eq:proof-con-search-gradient-ensemble}
    \ell(\x^\out) - \ell(\xs) \le \ell(\x^{\is}) - \ell(\xs) + \O\sbr{\textsc{Err}_N\sbr{\frac{T}{N}}}. 
\end{equation}
Next we perform the following case-by-case study for $\bar{d}_0 \define \max\{d_0, d_\epsilon\}$ where $d_0 \define \norm{\x^0 - \xs}$. 
The grid search upper bound is $d_{\max} \define \max\{d_\epsilon,\norm{\gh^0}T^2/L_\epsilon\}$, hence $N=\O\sbr{ \log \frac{T\norm{\gh^0}}{L_\epsilon d_\epsilon} }$. And we define $\bar{L}_\ell \define \max\{L_\ell, L_\epsilon\}$.

\paragraph{Case of $\bar{d}_0> d_{\max}$.} Then $\frac{\bar{L}_\ell \bar{d}_0^2}{T^2} \ge \norm{\gh^0}d_0$. Let $\is=0$, with probability at least $1-2\delta$:
\begin{align*}
    \ell(\x^\out) - \ell(\xs) &\overset{\eqref{eq:proof-con-search-gradient-ensemble}}{\le} \ell(\x^0) - \ell(\xs) + \O\sbr{\textsc{Err}_N\sbr{\frac{T}{N}}} \\
    &\overset{\eqref{eq:proof-con-search-gradient-x0}}{\le} \norm{\gh^0 } d_0 + \O\sbr{ \frac{\Delta(\x^0)d_0\sqrt{\log_+\frac{1}{\delta}}}{\sqrt{T}} + \textsc{Err}_N\sbr{\frac{T}{N}} } \\
    &\le \O\sbr{ \frac{\bar{L}_\ell  \bar{d}_0^2}{T^2} + \frac{\Delta(\x^0)d_0\sqrt{\log_+\frac{1}{\delta}}}{\sqrt{T}} + \textsc{Err}_N\sbr{\frac{T}{N}} }.
\end{align*}

\paragraph{Case of $\bar{d}_0\in [d_\epsilon, d_{\max}]$.} Let $\is = \lceil\log_2 \frac{\bar{d}_0}{d_\epsilon}\rceil$, then $D_{\is}/2 \le \bar{d}_0 \le D_{\is} = d_\epsilon 2^{\is}$, and allocated oracle budget $T_{\is} = \lfloor \frac{T}{2\is(1+\ln N)} \rfloor$. 
Applying~\pref{lemma:con-base-bounded} ($L_\nu=L_\ell,\nu=1$), with probability at least $1-2\delta$:
\begin{align*}
    \ell(\x^\out) - \ell(\xs) \overset{\eqref{eq:proof-con-search-gradient-ensemble}}{\le} {} & \ell(\x^{\is}) - \ell(\xs) + \O\sbr{ \textsc{Err}_N\sbr{\frac{T}{N}} } \\
    \le {} & \O\sbr{ \frac{L_\ell D_{\is}^2}{T_{\is}^2} + \frac{\theta_{T_{\is},\delta}\Delta_{D_{\is}} D_{\is}}{\sqrt{T_{\is}}} + \textsc{Err}_N\sbr{\frac{T}{N}}} \\
    = {} & \O\sbr{ \frac{L_\ell \bar{d}_0^2}{T^2}\sbr{\log\frac{\bar{d}_0}{d_\epsilon}}^2 + \frac{\Delta_{2\bar{d}_0}\bar{d}_0}{\sqrt{T}}\sbr{\log\frac{\bar{d}_0}{d_\epsilon}}^{\frac{1}{2}}\sbr{\log\frac{1}{\delta}} + \textsc{Err}_N\sbr{\frac{T}{N}} },
\end{align*}
where $\theta_{T,\delta}\define \log\frac{\log T}{\delta}$, and $\O(\cdot)$ omits double logarithmic factors.

Combining the above two cases, with probability at least $1-2\delta$:
\begin{equation*}
    \ell(\x^\out) - \ell(\xs)
    \le \O\sbr{ \frac{\bar{L}_\ell \bar{d}_0^2}{T^2}\sbr{\log\frac{\bar{d}_0}{d_\epsilon}}^2 + \frac{\Delta_{2\bar{d}_0}\bar{d}_0}{\sqrt{T}}\sbr{\log\frac{\bar{d}_0}{d_\epsilon}}^{\frac{1}{2}}\sbr{\log\frac{1}{\delta}} + \textsc{Err}_N\sbr{\frac{T}{N}} },
\end{equation*}
where $N=\O\sbr{ \log \frac{T\norm{\gh^0}}{L_\epsilon d_\epsilon} }$.
\end{proof}
\subsection{Proof of~\pref{thm:con-search-value}}
\label{appendix:proof-con-search-value}

\begin{proof}
Let $\textsc{Err}_{N}(M)$ be ensemble error given $N$ candidates with $M$ oracle budget for each, with probability at least $1-\delta$, as stated in~\pref{subsec:selection}. That means with probability at least $1-\delta$, for any $0\le \is \le N$,
\begin{equation}
    \label{eq:proof-con-search-value-ensemble}
    \ell(\x^\out) - \ell(\xs) \le \ell(\x^{\is}) - \ell(\xs) + \O\sbr{\textsc{Err}_N\sbr{\frac{T}{N}}}. 
\end{equation}
Next we perform the following case-by-case study for $\bar{d}_0 \define \max\{d_0, d_\epsilon\}$ where $d_0 \define \norm{\x^0 - \xs}$. 
The grid search upper bound is $d_{\max} \define \max\{1, d_\epsilon,(\ellh^0 - \ell_\epsilon^\star)T^2/L_\epsilon\}$, hence $N=\O\sbr{ \log \frac{T(\ellh^0 - \ell_\epsilon^\star)}{L_\epsilon d_\epsilon} }$. And we define $\bar{L}_\nu \define \max\{L_\nu, L_\epsilon\}$.

\paragraph{Case of $\bar{d}_0 > d_{\max}$.} Then $\frac{\bar{L}_\nu \bar{d}_0^{1+\nu}}{T^{\frac{1+3\nu}{2}}} \ge \frac{L_\epsilon \bar{d}_0}{T^2} > \ellh^0 - \ell_\epsilon^\star$. Let $\is = 0$, with probability at least $1-2\delta$,
\begin{align*}
    \ell(\x^\out) - \ell(\xs) &\overset{\eqref{eq:proof-con-search-value-ensemble}}{\le} \ell(\x^0) - \ell(\xs) + \O\sbr{ \textsc{Err}_N\sbr{\frac{T}{N}}} \\
    &\le \ellh^0 - \ell_\epsilon^\star + \O\sbr{ \frac{\sigma_0\sqrt{\log(1/\delta)}}{\sqrt{T}} + \textsc{Err}_N\sbr{\frac{T}{N}} } \\
    &\le \O\sbr{ \frac{\bar{L}_\nu \bar{d}_0^{1+\nu}}{T^\frac{1+3\nu}{2}} + \frac{\sigma_0\sqrt{\log(1/\delta)}}{\sqrt{T}} + \textsc{Err}_N\sbr{\frac{T}{N}} },
\end{align*}
where the second inequality uses~\pref{lemma:ensemble} with $N=1,M=\frac{T}{8}$, and define $\sigma_0$ as the maximum function value noise at $\x^0$, i.e. $\Pr[|{\ell(\x^0) - \ellt(\x^0)}| \le \sigma_0] = 1$. Moreover, we can decompose $\sigma_0$ by:
\begin{equation*}
    \abs{ \ell(\x^0) -  \ellt(\x^0) }
    \le \abs{( \ell(\x^0) - \ellt(\x^0)) - (\ell(\xs) - \ellt(\xs))} + \abs{\ell(\xs) - \ellt(\xs)}
    \le d_0 \Delta_{d_0} + \abs{\ell(\xs) - \ellt(\xs)},
\end{equation*}
where we use the fact that $\norm{\nabla [\ell(\x) - \ellt(\x)]} = \norm{\nabla \ell(\x) - \g(\x)} \le \Delta(\x)$, as defined in~\pref{subsec:setups}. Therefore,
\begin{equation*}
    \frac{\sigma_0\sqrt{\log(1/\delta)}}{\sqrt{T}} \le \frac{d_0\Delta_{d_0}\sqrt{\log(1/\delta)}}{\sqrt{T}} + \frac{\sigma_\star\sqrt{\log(1/\delta)}}{\sqrt{T}}.
\end{equation*}

\paragraph{Case of $\bar{d}_0 \in [d_\epsilon, d_{\max}]$.} Let $\is = \lceil\log_2 \frac{\bar{d}_0}{d_\epsilon}\rceil$, then $D_{\is}/2 \le \bar{d}_0 \le D_{\is} = d_\epsilon 2^{\is}$, and allocated oracle budget $T_{\is} = \lfloor \frac{T}{2\is(1+\ln N)} \rfloor$. 
Applying~\pref{lemma:con-base-bounded}, with probability at least $1-2\delta$,
\begin{align*}
    \ell(\x^\out) - \ell(\xs) \overset{\eqref{eq:proof-con-search-value-ensemble}}{\le} {} & \ell(\x^{\is}) - \ell(\xs) + \O\sbr{ \textsc{Err}_N\sbr{\frac{T}{N}} } \\
    \le {} & \O\sbr{ \frac{L_\nu D_{\is}^{1+\nu}}{T_{\is}^{\frac{1+3\nu}{2}}} + \frac{\theta_{T_{\is},\delta}\Delta_{D_{\is}}D_{\is} }{\sqrt{T_{\is}}} + \textsc{Err}_N\sbr{\frac{T}{N}} } \\
    = {} & \O\sbr{ \frac{L_\nu \bar{d}_0^{1+\nu}}{T^{\frac{1+3\nu}{2}}}\sbr{\log\frac{\bar{d}_0}{d_\epsilon}}^{\frac{1+3\nu}{2}} + \frac{\theta_{T,\delta}\Delta_{2\bar{d}_0}\bar{d}_0 }{\sqrt{T}}\sbr{\log \frac{\bar{d}_0}{d_\epsilon}}^{\frac{1}{2}} + \textsc{Err}_N\sbr{\frac{T}{N}} },
\end{align*}
where $\theta_{T,\delta}\define \log\frac{\log T}{\delta}$, and $\O(\cdot)$ omits double logarithmic factors.

Combining the above two cases, with probability at least $1-2\delta$:
\begin{equation*}
    \ell(\x^\out) - \ell(\xs)
    \le \O\sbr{ \frac{\bar{L}_\nu \bar{d}_0^{1+\nu}}{T^{\frac{1+3\nu}{2}}}\sbr{\log\frac{\bar{d}_0}{d_\epsilon}}^{\frac{1+3\nu}{2}} + \frac{\Delta_{2\bar{d}_0}\bar{d}_0 }{\sqrt{T}}\sbr{\log \frac{\bar{d}_0}{d_\epsilon}}^{\frac{1}{2}}\sbr{\log\frac{1}{\delta}} + \frac{\sigma_\star \sqrt{\log(1/\delta)}}{\sqrt{T}} + \textsc{Err}_N\sbr{\frac{T}{N}} },
\end{equation*}
where $N=\O \sbr{\log \frac{T(\ellh^0 - \ell_\epsilon^\star)}{L_\epsilon d_\epsilon} }$.
\end{proof}

\subsection{Proof of \pref{lemma:selection-variance}}
\label{app:selection-variance}
\begin{proof}
    Let $i \in [N]$. Using a two-sided Bernstein's inequality (e.g., using Proposition 2.14 of \citet{wainwright2019high}), for any $t > 0$, denoting by $\ell^i=\ell(\x_i)$ and $\ellh^i=\frac1M \sum_{j=1}^M \ellt_j(\x_i)$, it holds that
    \begin{equation*}
        \Pr\mbr{\abs{\ellh^i-\ell^i} \geq t} \leq 2 \exp\sbr{-\frac{M t^2}{2 \var[\ellt(\x_i) \mid \x_i] + \frac23 \sigma_\ell t}},
    \end{equation*}
    Hence, for any $\delta \in (0,1)$, setting $t=\sqrt{2 \var[\ellt(\x_i) \mid \x_i] \log(2/\delta)/M}+\frac{2}{3M} \sigma_\ell \log(2/\delta)$, with probability at least $1-\delta$,
    \begin{equation*}
        \abs{\ellh^i - \ell^i} < \frac{\sqrt{2 \var[\ellt(\x_i) \mid \x_i]\log(2/\delta)}}{\sqrt{M}} + \frac{2 \sigma_\ell \log(2/\delta)}{3M}.
    \end{equation*}
    Thus, performing a union bound over $i \in [N]$ and setting $\delta=\delta'/N$, with probability at least $1-\delta'$, for all $i \in [N]$,
    \begin{align*}
        \abs{\ellh^i - \ell^i}
        &<
        \sqrt{\frac{2 V_0 \log(2N/\delta')}{M}+\frac{4 V_1^2 \log^2(2N/\delta')}{M^2} + \frac{(\ell^i-\ell^\star)^2}{4}} + \frac{2 \sigma_\ell \log(2 N/\delta')}{3M}.
        \\&\leq
        \frac12 (\ell^i-\ell^\star) + \frac{\sqrt{2 V_0 \log(2 N/\delta')}}{\sqrt{M}}
        + \frac{(2 \sigma_\ell + 6 V_1) \log(2 N/\delta')}{3M}.
    \end{align*}
    The rest of the argument will be conditioned on the above high-probability event, under which it holds that
    \[
    \ellh^i - \ell^\star \leq \frac32 (\ell^i - \ell^\star) + \frac{\sqrt{2 V_0 \log(2 N/\delta')}}{\sqrt{M}} + \frac{(2 \sigma_\ell + 6 V_1)\log(2 N/\delta')}{3M}
    \]
    and
    \[
    \ell^i - \ell^\star \leq 2 (\ellh^i - \ell^\star) + \frac{\sqrt{8 V_0 \log(2 N/\delta')}}{\sqrt{M}} + \frac{(4 \sigma_\ell + 12 V_1)\log(2 N/\delta')}{3M}
    .
    \]
    Thus, for $\hat{i} \in \argmin_{i \in [N]} \ellh^i$ and $i_\star \in \argmin_{i \in [N]} \ell^i$,
    \begin{align*}
        \ell^{\hat{i}} - \ell^\star &\leq
        2 (\ellh^{\hat{i}} - \ell^\star) + \frac{\sqrt{8 V_0 \log(2N/\delta')}}{\sqrt{M}} + \frac{(4 \sigma_\ell + 12 V_1)\log(2N/\delta')}{3M}
        \\&
        \leq
        2 (\ellh^{i_\star} - \ell^\star) + \frac{\sqrt{8 V_0 \log(2N/\delta')}}{\sqrt{M}} + \frac{(4 \sigma_\ell + 12 V_1)\log(2N/\delta')}{3M}
        \\&\leq
        3 (\ell^{i_\star} - \ell^\star)
        + \frac{\sqrt{32 V_0 \log(2N/\delta')}}{\sqrt{M}} + \frac{(8 \sigma_\ell + 24 V_1)\log(2N/\delta')}{3M}
        ,
    \end{align*}
    which finishes the proof.
\end{proof}

\subsection{Universal Convergence with Less Prior Knowledge but Worse Ensemble Error}
\label{app:less-prior}
In this part, we prove that \emph{without} the prior knowledge of the lower bound $\ell(\xs) \ge \ell_\epsilon^\star$, we can still achieve universal convergence to \Holder smoothness, while with a larger ensemble error due to the increased number of base algorithms, which leads to an $\Ot(\sigma_\ell /T^{1/4})$ term in the worst case.    
\begin{algorithm}[!t]
    \caption{\textsc{Grasp-C} without Lower Bound of Function Value}
    \label{alg:grid-search-II}
    \begin{algorithmic}[1]
    \REQUIRE Oracle budget $T$, initial point $\x^0$, $d_\epsilon > 0$.
    \STATE Set $N = \lceil \sqrt{T} \rceil$
    \STATE Independently sample $\ellt(\x^0)$ for $\frac{T}{2(N+1)}$ times, calculate average $\ellh^0 = \frac{2(N+1)}{T}\sum_{t=1}^{T/(2(N+1))}\ellt_t(\x^0)$
    \FOR{$i=1,2,\dots,N$}
        \STATE Run \UniXGrad~(see~\pref{lemma:con-base-bounded}) with initial point $\x_1=\x^0$, domain diameter $D_i=d_\epsilon 2^i$, Oracle budget $T_i=\lfloor \frac{T}{i^2\pi^2/3}\rfloor$, and receive the output $\x^i$ from the algorithm
        \STATE Independently sample $\ellt(\x^i)$ for $\frac{T}{2(N+1)}$ times, calculate average $\ellh^i = \frac{2(N+1)}{T}\sum_{t=1}^{T/(2(N+1))}\ellt_t (\x^i)$
    \ENDFOR
    \ENSURE $\x^\out = \x^{\is}$ with $\is = \argmin_{0\le i\le N} \ellh^i$.
    \end{algorithmic}
\end{algorithm}
\begin{myThm}
    \label{thm:con-search-T/i2}
    Under Assumptions~\ref{assum:bounded-gradient-oracle} and~\ref{assum:bounded-value-oracle}, and assume the objective $\ell(\x)$ is $(L_\nu,\nu)$-\Holder smooth, Algorithm~\ref{alg:grid-search-II} ensures that, with probability at least $1-2\delta$:
    \begin{equation*}
        \ell(\x^\out) - \ell(\xs) \le \O\sbr{ \frac{L_\nu \bar{d}_0^{1+\nu}}{T^{\frac{1+3\nu}{2}}}\sbr{\log\frac{\bar{d}_0}{d_\epsilon}}^{1+3\nu} + \frac{\Delta_{2\bar{d}_0} \bar{d}_0}{\sqrt{T}}\sbr{\log\frac{\bar{d}_0}{d_\epsilon}}\sbr{\log\frac{1}{\delta}} + \textsc{Err}_{\sqrt{T}}\sbr{\sqrt{T}} },
    \end{equation*}
    where $\bar{d}_0 \define \max\{d_0, d_\epsilon\}$, $\Delta_{2\bar{d}_0}$ is the maximum noise defined in Eq.~\eqref{eq:variance-bound}, $\textsc{Err}$ is the ensemble error specified in Section~\ref{subsec:selection}.
\end{myThm}
\begin{proof}
Set the total number of base algorithms to be $N=\lceil \sqrt{T}\rceil$, and oracle budget $T_i=\lfloor \frac{T}{i^2\pi^2/3}\rfloor$ for the $i$-th base algorithm. 
We can verify that \smash{$\sum_{i=1}^N T_i \le \frac{T}{\pi^2/3}\sum_{i=1}^{\infty}\frac{1}{i^2} = \frac{T}{2}$}. 
Let $\textsc{Err}_{N}(M)$ be ensemble error given $N$ candidates with $M$ oracle budget for each, with probability at least $1-\delta$, as stated in~\pref{subsec:selection}. That means with probability at least $1-\delta$, for any $0\le \is \le N$,
\begin{equation}
    \label{eq:proof-con-search-T/i2-ensemble}
    \ell(\x^\out) - \ell(\xs) \le \ell(\x^{\is}) - \ell(\xs) + \O\sbr{\textsc{Err}_N\sbr{\frac{T}{N}}}. 
\end{equation}
Now we perform the following case-by-case study for $\bar{d}_0 \define \max\{d_0, d_\epsilon\}$, with $d_0 \define \norm{\x^0 - \xs}$ and any $d_\epsilon>0$.
\paragraph{Case of $\bar{d}_0>d_\epsilon 2^N$.} Then $\log_2 \frac{\bar{d}_0}{d_\epsilon}>N\ge \sqrt{T}$. Let $\is=0$, with probability at least $1-\delta$:
\begin{align*}
    \ell(\x^\out) - \ell(\xs) &\overset{\eqref{eq:proof-con-search-T/i2-ensemble}}{\le} \ell(\x^0) - \ell(\xs) + \O\sbr{ \textsc{Err}_N\sbr{\frac{T}{N}} }
    \le L_\nu d_0^{1+\nu} + \O\sbr{ \textsc{Err}_N\sbr{\frac{T}{N}} } \\
    &\le \O\sbr{ \frac{L_\nu d_0^{1+\nu}}{T^{\frac{1+3\nu}{2}}}\sbr{ \log \frac{\bar{d}_0}{d_\epsilon} }^{1+3\nu} + \O\sbr{ \textsc{Err}_N\sbr{\frac{T}{N}} }  },
\end{align*}
where the second inequality is because $\nabla \ell(\xs)=\mathbf{0}$ and $\ell(\x) - \ell(\xs)\le \norm{\nabla \ell(\x) - \nabla \ell(\xs)}\norm{\x - \xs}\le L_\nu\norm{\x - \xs}^{1+\nu}$ for all $\x\in\X$, and the third inequality is because $\log_2 \frac{\bar{d}_0}{d_\epsilon}\ge \sqrt{T}$ by assumption.

\paragraph{Case of $\bar{d}_0 \in [d_\epsilon, d_\epsilon 2^N]$.} Let $\is = \lceil\log_2 \frac{\bar{d}_0}{d_\epsilon}\rceil$, then $D_{\is}/2 \le \bar{d}_0 \le D_{\is} = d_\epsilon 2^{\is}$, and allocated oracle budget $T_{\is} = \lfloor \frac{T}{\is^2\pi^2/3}\rfloor$. 
Applying~\pref{lemma:con-base-bounded}, with probability at least $1-2\delta$,
\begin{align*}
    \ell(\x^\out) - & \ell(\xs) \overset{\eqref{eq:proof-con-search-T/i2-ensemble}}{\le} \ell(\x^{\is}) - \ell(\xs) + \O\sbr{ \textsc{Err}_N\sbr{\frac{T}{N}} } \\
    \le {} & \O\sbr{ \frac{L_\nu D_{\is}^{1+\nu}}{T_{\is}^{\frac{1+3\nu}{2}}} +  \frac{\theta_{T_{\is},\delta}\Delta_{D_{\is}}D_{\is} }{\sqrt{T_{\is}}} + \textsc{Err}_N\sbr{\frac{T}{N}} } \\
    = {} & \O\sbr{ \frac{L_\nu \bar{d}_0^{1+\nu}}{T^{\frac{1+3\nu}{2}}}\sbr{\log\frac{\bar{d}_0}{d_\epsilon}}^{1+3\nu} + \frac{\Delta_{2\bar{d}_0} \bar{d}_0}{\sqrt{T}}\sbr{\log\frac{\bar{d}_0}{d_\epsilon}}\sbr{\log\frac{1}{\delta}} + \textsc{Err}_N\sbr{\frac{T}{N}} }.
\end{align*}

Combining the above two cases, with probability at least $1-2\delta$:
\begin{equation*}
    \ell(\x^\out) - \ell(\xs) \le \O\sbr{ \frac{L_\nu \bar{d}_0^{1+\nu}}{T^{\frac{1+3\nu}{2}}}\sbr{\log\frac{\bar{d}_0}{d_\epsilon}}^{1+3\nu} + \frac{\Delta_{2\bar{d}_0} \bar{d}_0}{\sqrt{T}}\sbr{\log\frac{\bar{d}_0}{d_\epsilon}}\sbr{\log\frac{1}{\delta}} + \textsc{Err}_N\sbr{\frac{T}{N}} },
\end{equation*}
which finishes the proof.
\end{proof}

\subsection{Omitted Details of Base Algorithm}
\label{app:cvx-base}
In this part, we introduce the base algorithm \UniXGrad, with its convergence rate analysis.

\begin{algorithm}[!t]
    \caption{\textsc{UniXGrad}~\citep{NeurIPS'19:UniXGrad}}
    \label{alg:unixgrad}
    \begin{algorithmic}[1]
    \REQUIRE Oracle budget $2T$, domain diameter $D$, initial point $\xh_1$, weight $\alpha_t = t$.
    \FOR{$t=1,2,\dots,T$}
        \STATE Set step size $\eta_t = \frac{2D}{\sqrt{1 + \sum_{s=1}^{t-1}\alpha_s^2\norm{\g(\xb_s) - \g(\xt_s)}^2}}$
        \STATE Update $\x_t = \argmin_{\x\in\X} \inner{\alpha_t\g(\xt_t)}{\x} + \frac{1}{2\eta_t}\norm{\x - \xh_{t}}^2$ with $\xt_t = \frac{1}{\sum_{s=1}^{t}\alpha_s}(\alpha_t \xh_t + \sum_{s=1}^{t-1}\alpha_s \x_s)$
        \STATE Update $\xh_{t+1} = \argmin_{\x\in\X} \inner{\alpha_t\g(\xb_t)}{\x} + \frac{1}{2\eta_t}\norm{\x - \xh_{t}}^2$ with $\xb_t = \frac{1}{\sum_{s=1}^{t}\alpha_s}(\alpha_t \x_t + \sum_{s=1}^{t-1}\alpha_s \x_s)$
    \ENDFOR
    \ENSURE $\xb_T$.
    \end{algorithmic}
\end{algorithm}

\begin{myLemma}[Base Algorithm with Bounded Domain]
    \label{lemma:con-base-bounded}
    Assuming that domain $\X$ is bounded by $D$, $\ell(\x)$ is $(L_\nu,\nu)$-Hölder smooth, and under~\pref{assum:bounded-gradient-oracle}, \UniXGrad~\citep{NeurIPS'19:UniXGrad}, i.e. Algorithm~\ref{alg:unixgrad}, when given $D$, ensures that with probability at least $1-\delta$:
\begin{equation*}
    \ell(\xb_T) - \min_{\x\in\X}\ell(\x) \le \O\sbr{ \frac{L_\nu D^{1+\nu}}{T^{\frac{1+3\nu}{2}}} + \frac{\theta_{T,\delta} \Delta_D D}{\sqrt{T}}},
\end{equation*}
where $\Delta_D$ is the maximum noise bound defined in~\pref{eq:variance-bound}, and $\theta_{T,\delta}\define \log\frac{\log T}{\delta}$.
\end{myLemma}
\begin{proof}[Proof of~\pref{lemma:con-base-bounded}]
    To begin with, we introduce the anytime online-to-batch conversion lemma~\citep{ICML'19:Cutkosky}, which is a key ingredient for the analysis of \UniXGrad.
    \begin{myLemma}[\citet{ICML'19:Cutkosky}]
        \label{lemma:anytime-O2B}
        If the objective function $\ell(\cdot)$ is convex, then it holds that:
        \begin{equation}
            \label{eq:our-conversion}
            \alpha_{1:T} \mbr{\ell(\xb_T) - \ell(\xs)} \le \sumT \inner{\alpha_t \nabla \ell(\xb_t)}{\x_t - \xs},
        \end{equation}
        where $\xb_t \define \frac{1}{\alpha_{1:t}} \sumt \alpha_s \x_s$, $\alpha_t>0$ for all $t\in[T]$, and $\alpha_{1:t} \define \sum_{s=1}^{t}\alpha_s$.
    \end{myLemma}
    By~\pref{lemma:anytime-O2B} with $\alpha_t = t$, we have:
    \begin{equation}
        \label{eq:unixgrad-o2b}
        \alpha_{1:T}\mbr{\ell(\xb_T) - \ell(\xs)} \le \sum_{t=1}^T \alpha_t\inner{\nabla \ell(\xb_t)}{\x_t - \xs}= \underbrace{\sum_{t=1}^{T}\alpha_t \inner{\nabla \ell(\xb_t) - \g(\xb_t)}{\x_t - \xs}}_{\textsc{Gap}_T} + \underbrace{\sum_{t=1}^{T}\alpha_t \inner{\g(\xb_t)}{\x_t - \xs}}_{\Reg_T}. 
    \end{equation}
    To bound $\textsc{Gap}_T$, we apply~\pref{lemma:weighted-concentration} with $X_t=\inner{\nabla \ell(\xb_t) - \g(\xb_t)}{\x_t-\xs}$, $\hat{X}_t=0$, and $c=\Delta_D D$ where $\Delta_D\triangleq\max_{\norm{\x}\le D}\Delta(\x)$, then with probability at least $1-\delta$, it holds that:
    \begin{equation}
        \label{eq:con-base-bounded-gap}
        \textsc{Gap}_T \le 8 T \sqrt{ \theta_{T,\delta}D^2\sum_{t=1}^{T}\norm{\nabla \ell(\xb_t) - \g(\xb_t)}^2 + \theta_{T,\delta}^2 \Delta_D^2 D^2 } \le \O\sbr{ \sqrt{\theta_{T,\delta}}\Delta_D  D T^{\frac{3}{2}} + \theta_{T,\delta}\Delta_D D T }.
    \end{equation}
    And for $\textsc{Reg}_T$, following the analysis of~\citet[Theorem 4]{NeurIPS'19:UniXGrad}, i.e., starting from their Eq.~(11):
    \begin{equation*}
        \Reg_T
        \le \frac{7D}{2}\sqrt{1 + \sum_{t=1}^{T} \alpha_t^2 \norm{\g(\xb_t) - \g(\xt_t)}^2 } - \frac{1}{2}\sum_{t=1}^{T} \frac{1}{\eta_{t+1}}\norm{\x_t - \xh_t}^2.
    \end{equation*}
    Next, we follow the same steps in~\citet[Theorem 2]{NeurIPS'25:GV4OPT}.
    By~\pref{lemma:inexact-smooth}, for any $\beta>0$, denote by $L = \beta^{\frac{\nu-1}{1+\nu}} L_\nu^{\frac{2}{1+\nu}}$:
    \begin{equation}
        \label{eq:unixgrad-inexact-smoothness}
        \norm{\nabla \ell(\x) - \nabla \ell(\y)}^2 \le L^2 \norm{\x - \y}^2 + 4L \beta.
    \end{equation}   
    Let $A_t \define \sum_{s=1}^{t}\alpha_s^2\norm{ \g(\xb_s) - \g(\xt_s) }^2$. WLOG assume $\sqrt{A_T}\ge 2LD$, otherwise we will finish the proof trivially. Define $t_0\in[T-1]$ that, if $\sqrt{A_1} > 2LD$, let $t_0 = 1$, otherwise let $t_0 = \min\{t\in[T-1],\sqrt{A_{t+1}} > 2LD\}$. Then we have $\sqrt{A_{t_0}} \le 2LD$, while for all $t_0+1 \le t \le T$ it holds that $\sqrt{A_t} > 2LD$. Continuing with the inequality:
    \begin{align*}
        \Reg_T
        &\le \frac{7D}{2}\sqrt{1 + A_T} - \frac{1}{2}\sum_{t=1}^{T} \frac{1}{\eta_{t+1}}\norm{\x_t - \xh_t}^2 \\
        &\le \frac{7D}{2}\sqrt{1 + A_{t_0}} + \frac{7D}{2}\sqrt{ \sum_{t=t_0+1}^{T} \alpha_t^2 \norm{\g(\xb_t) - \g(\xt_t)}^2 } - \frac{1}{2}\sum_{t=1}^{T} \frac{1}{\eta_{t+1}}\norm{\x_t - \xh_t}^2\\
        &\le \frac{7D}{2}(1 + 2LD) + \frac{7D}{2}\sqrt{ \sum_{t=t_0+1}^{T} \alpha_t^2 \norm{\g(\xb_t) - \nabla \ell(\xb_t) + \nabla \ell(\xb_t) - \nabla \ell(\xt_t) + \nabla \ell(\xt_t) - \g(\xt_t)}^2 } - \frac{1}{2}\sum_{t=1}^{T} \frac{1}{\eta_{t+1}}\norm{\x_t - \xh_t}^2 \\
        &\le \frac{7D}{2}(1 + 2LD)  + \frac{7D}{2}\Delta_D T^{\frac{3}{2}} + \frac{7D}{2}\sqrt{\sum_{t=t_0+1}^{T} \alpha_t^2 \norm{\nabla \ell(\xb_t) - \nabla \ell(\xt_t)}^2}   - \frac{1}{2}\sum_{t=1}^{T} \frac{1}{\eta_{t+1}}\norm{\x_t - \xh_t}^2, 
    \end{align*}
    where we use the definition of $t_0$. Then we apply~\pref{eq:unixgrad-inexact-smoothness}, with the definition $\alpha_{1:t} = \sum_{s=1}^{t}\alpha_s$,
    \begin{align*}
        &\frac{7D}{2}\sqrt{\sum_{t=t_0+1}^{T} \alpha_t^2 \norm{\nabla \ell(\xb_t) - \nabla \ell(\xt_t)}^2}   - \frac{1}{2}\sum_{t=1}^{T} \frac{1}{\eta_{t+1}}\norm{\x_t - \xh_t}^2 \\
        \overset{\eqref{eq:unixgrad-inexact-smoothness}}{\le} {} & \frac{7D}{2}\sqrt{\sum_{t=t_0+1}^{T} \alpha_t^2L^2\norm{\xb_t - \xt_t}^2 + 4L\beta T^3  }   - \frac{1}{2}\sum_{t=1}^{T} \frac{1}{\eta_{t+1}}\norm{\x_t - \xh_t}^2 \\
        \le{} & \frac{7D}{2}\sqrt{\sum_{t=t_0+1}^{T} \frac{\alpha_t^4 L^2}{\alpha_{1:t}^2}\norm{\x_t - \xh_t}^2 } - \frac{1}{2}\sum_{t=1}^{T} \frac{1}{\eta_{t+1}}\norm{\x_t - \xh_t}^2  + 7D\sqrt{L\beta T^3 },
    \end{align*}
    where in the last line we use the definitions of $\xb_t$ and $\xt_t$. Since $\alpha_t = t$, we have $\frac{\alpha_t^4}{\alpha_{1:t}^2} = \frac{4t^4}{t^2(1+t)^2}\le 4$, then
    \begin{align*}
        & \frac{7D}{2}\sqrt{\sum_{t=t_0+1}^{T} \frac{\alpha_t^4 L^2}{\alpha_{1:t}^2}\norm{\x_t - \xh_t}^2 } - \frac{1}{2}\sum_{t=1}^{T} \frac{1}{\eta_{t+1}}\norm{\x_t - \xh_t}^2 \\
        \le{} & 7D\sqrt{\sum_{t=t_0+1}^{T} L^2\norm{\x_t - \xh_t}^2 } - \frac{1}{2}\sum_{t=1}^{T} \frac{1}{\eta_{t+1}}\norm{\x_t - \xh_t}^2 \\
        \le{} & \frac{49LD^2}{2} + \frac{1}{2} \sum_{t=t_0+1}^{T} \sbr{ L - \frac{1}{\eta_{t+1}}} \norm{\x_t - \xh_t}^2 \\
        ={} & \frac{49LD^2}{2} + \frac{1}{2} \sum_{t=t_0+1}^{T} \sbr{ L - \frac{\sqrt{1 + A_t}}{2D} } \norm{\x_t - \xh_t}^2 \le \frac{49LD^2}{2},
    \end{align*}
    where in the second inequality we use $\sqrt{ab}\le a/2+b/2$ for $a,b>0$, and in the last line we use the definition of $\eta_{t+1} = \frac{2D}{\sqrt{1 + A_t}}$, and the definition of $t_0$ such that for all $t> t_0$, $\sqrt{A_t} \ge 2LD$. Combining everything together, we have
    \begin{equation}
        \label{eq:con-base-bounded-reg}
        \Reg_T \le \O\sbr{ LD^2 + D\sqrt{L\beta T^3} + \Delta_D D T^{\frac{3}{2}} }.
    \end{equation}
Setting $\beta = L_\nu D^{1+\nu}T^{\frac{-(3+3\nu)}{2}}$, substituting~\pref{eq:con-base-bounded-gap} and~\pref{eq:con-base-bounded-reg} into~\pref{eq:unixgrad-o2b}, with probability at least $1-\delta$, it holds that:
    \begin{equation*}
        \ell(\xb_T) - \ell(\xs) \le \O\sbr{ \frac{L_\nu D^{1+\nu}}{T^{\frac{1+3\nu}{2}}} + \frac{\theta_{T,\delta} \Delta_D D}{\sqrt{T}}},
    \end{equation*}
    which finishes the proof.
\end{proof}


\section{Supporting Lemmas}
\label{app:supporting}
In this section, we provide supporting lemmas for the main results.
\begin{myLemma}[Lemma 9 of \citet{ICML'24:Amit}]
    \label{lemma:ensemble}
    Let $\ell:\X\to \R$ and $\ellt$ a zeroth-order oracle of $\ell$ such that for all $\x\in\X$, $\E[\ellt(\x) \given \x] = \ell(\x)$ and $|\ellt(\x) - \ell(x)|\le \sigma_\ell$ with some $\sigma_\ell>0$. Given candidates $\x_1,\cdots, \x_N$, let $\xb = \argmin_{i\in[N]}\sum_{j=1}^{M}\ellt_j(\x_i)$, where $\ellt_j(\x_i)$ means the $j$-th independent function evaluation at $\x_i$. Then for any $\delta\in(0,1)$, with probability at least $1-\delta$:
    \begin{equation*}
        \ell(\xb) \le \min_{i\in[N]} \frac{1}{M}\sum_{j=1}^{M}\ellt_j(\x_i) + \sqrt{\frac{2\sigma_\ell^2 \log \frac{2N}{\delta}}{M}} \le \min_{i\in[N]} \ell(\x_i) + \sqrt{\frac{8\sigma_\ell^2 \log \frac{2N}{\delta}}{M}}.
    \end{equation*}
\end{myLemma}

\begin{myLemma}[Sample Gradient]
    \label{lemma:sample-gradient}
    Under~\pref{assum:bounded-gradient-oracle}, for any $\x\in\X$, independently sample $\g(\x)$ for $M$ times, and define $\gh = \frac{1}{M}\sum_{t=1}^{M}\g_t(\x)$ where $\g_t(\x)$ is the $t$-th gradient query on $\x$. Then for any $\delta\in(0,1)$, with probability at least $1-\delta$:
    \begin{equation*}
    \norm{\gh - \nabla \ell(\x)} \le \O\sbr{ \frac{\Delta(\x)\sqrt{\log_+\frac{1}{\delta}}}{\sqrt{M}} }.
    \end{equation*}
\end{myLemma}

\begin{proof}[Proof of~\pref{lemma:sample-gradient}]
    By~\pref{lemma:vanilla-concentration}, for any $\epsilon>0$:
    \begin{equation*}
        \Pr\mbr{ \norm{\gh - \nabla \ell(\x)} \ge \frac{\sqrt{2}(1+\epsilon)\Delta(\x)}{\sqrt{M}} } = \Pr\mbr{ \left\| \sum_{t=1}^{M}(\g_t(\x) - \nabla \ell(\x)) \right\| \ge \sqrt{2}(1+\epsilon)\Delta(\x)\sqrt{M} } \le \exp(-\epsilon^2 / 3).
    \end{equation*}
    Solving $\delta=\exp(-\epsilon^2/3)$, then with probability at least $1-\delta$:
    \begin{equation*}
        \norm{\gh - \nabla \ell(\x)} \le \O\sbr{ \frac{\Delta(\x)\sqrt{\log_+\frac{1}{\delta}}}{\sqrt{M}} },
    \end{equation*}
    which completes the proof.
\end{proof}

\begin{myLemma}[Lemma 7 of \citet{ICML'23:DoG}]
    \label{lemma:weighted-concentration}
    Let $\{\alpha_i\}_{i=1}^{\infty}$ be non-negative and non-decreasing sequence. let $X_t$ be a martingale difference sequence adapted to $\F_t$ such that $|X_t|\le c$ with constant $c>0$. Then for all $\delta\in(0,1)$, and $\hat{X}_t\in\F_{t-1},|\hat{X}_t|\le c$, it holds that
    \begin{equation}
        \Pr\left[ \exists t\in[T]: \left| \sum_{i=1}^{t} \alpha_i X_i \right| \ge 8 \alpha_t \sqrt{\theta_{t,\delta} \sum_{i=1}^{t}\sbr{X_i - \hat{X}_i}^2 + c^2 \theta_{t,\delta}^2 } \right] \le \delta,
    \end{equation}
where $\theta_{t,\delta}\triangleq \log\frac{60\log(6t)}{\delta}$.
\end{myLemma}

\begin{myLemma}[Lemma 2.3 of \citet{ghadimi2013stochastic}]
    \label{lemma:vanilla-concentration}
    Let $X_1,\dots,X_n\in\R^d$ be a martingale difference sequence with respect to $\F_1,\dots,\F_n$. Assuming $\E[\exp(\frac{\norm{X_i}^2}{\sigma^2})|\F_1,\dots,\F_{i-1}]\le e$, for any $\epsilon>0$:
    \begin{equation}
        \Pr\mbr{ \left\| \sum_{i=1}^{n} X_i \right\| \ge \sqrt{2}(1+\epsilon)\sigma\sqrt{n} } \le \exp(-\epsilon^2/3).
    \end{equation}
\end{myLemma}

The following lemma is a direct combination of~\citet[Lemma 1]{nesterov2015universal} and~\citet[Theorem 1]{devolder2014first}.
\begin{myLemma}
    \label{lemma:inexact-smooth}
    Suppose the function $f$ is $(L_\nu, \nu)$-Hölder smooth. Then, for any $\delta>0$, denoting by $L = \delta^{\frac{\nu-1}{1+\nu}} L_\nu^{\frac{2}{1+\nu}}$, it holds that for all $\x,\y\in\R^d$:
    \begin{equation}
        \label{eq:inexact-smooth}
        \norm{\nabla f(\x) - \nabla f(\y)}^2 \le L^2 \norm{\x - \y}^2 + 4L \delta.
    \end{equation}
\end{myLemma}

\end{document}